\documentclass[11pt]{article}
\usepackage{amsmath,amssymb,amsthm}
\usepackage{latexsym}
\usepackage[dvipdfmx]{graphicx}
\usepackage{algorithm}
\usepackage{algorithmic}
\usepackage{authblk}

\def\Nbb{\mathbb{N}}
\def\Rbb{\mathbb{R}}
\def\Ebb{\mathbb{E}}
\def\st{\text{\rm s.\,t.\ }}

\newtheorem{theorem}{Theorem}
\newtheorem{lemma}{Lemma}
\usepackage{color}

\title{Breakdown Point of Robust Support Vector Machine}

\author[1]{Takafumi Kanamori}
\author[2]{Shuhei Fujiwara}
\author[2]{Akiko Takeda}

\affil[1]{Nagoya University}
\affil[2]{The University of Tokyo}

\date{}

\begin{document}
\maketitle

\begin{abstract}
 The support vector machine (SVM) is one of the most successful learning methods for solving classification
 problems. Despite its popularity, SVM has a serious drawback, that is sensitivity to outliers in training samples. The
 penalty on misclassification is defined by a convex loss called the hinge loss, and the unboundedness of the convex
 loss causes the sensitivity to outliers. To deal with outliers, robust variants of SVM have been proposed, such as the
 robust outlier detection algorithm and an SVM with a bounded loss called the ramp loss. In this paper, we propose a
 robust variant of SVM and investigate its robustness in terms of the breakdown point. The breakdown point is a
 robustness measure that is the largest amount of contamination such that the estimated classifier still gives
 information about the non-contaminated data. The main contribution of this paper is to show an exact evaluation of the
 breakdown point for the robust SVM. 
 For learning parameters such as the regularization parameter in our algorithm, 
 we derive a simple formula that guarantees the robustness of the classifier. 
 When the learning parameters are determined with a grid search using cross validation, our formula works to reduce the
 number of candidate search points. The robustness of the proposed method is confirmed in numerical experiments. We show
 that the statistical properties of the robust SVM are well explained by a theoretical analysis of the breakdown point. 
\end{abstract}

\section{Introduction}
\label{sec:Introduction} 
Support vector machine (SVM) is a highly developed classification method that is widely used in real-world data
analysis~\cite{Cortes95support-vectornetworks,book:Schoelkopf+Smola:2002}. 
The most popular implementation is called $C$-SVM, which uses the maximum margin criterion with a penalty for
misclassification. 
The positive parameter $C$ tunes the balance between the maximum margin and penalty. 
As a result, the classification problem can be formulated as a convex quadratic problem based on training data. 
A separating hyper-plane for classification is obtained from the optimal solution of the problem. 
Furthermore, complex non-linear classifiers are obtained by using the reproducing kernel
Hilbert space (RKHS) as a statistical model of the classifiers~\cite{berlinet04:_reprod_hilber}. 
There are many variants of SVM for solving binary classification problems, 
such as $\nu$-SVM, E$\nu$-SVM, least square 
SVM~\cite{AdvLT:Perez-Cruz+etal:2003,NECO:Scholkopf+etal:2000,Suykens:1999:LSS:326394.326408}. 
Moreover, the generalization ability of SVM has been analyzed in many
studies~\cite{bartlett06:_convex_class_risk_bound,JMLR:Steinwart:2001,zhang04:_statis}.  

In practical situations, however, SVM has drawbacks. The remarkable feature of
the SVM is that the separating hyperplane is determined mainly from misclassified
samples. Thus, the most misclassified samples significantly affect the classifier, meaning 
that the standard SVM is extremely fragile to the presence of outliers. 
In $C$-SVM, the penalties of sample points are measured in terms of the hinge loss, which is a convex surrogate of the
0-1 loss for misclassification. 
The convexity of the hinge loss causes SVM to be unstable in the presence of outliers, 
since the convex function is unbounded and puts an extremely large penalty on outliers. 
One way to remedy the instability is to replace the convex loss 
with a non-convex bounded loss to suppress outliers. 
Loss clipping is a simple method to obtain a bounded loss from a convex loss~\cite{shen03:_learn,NIPS:Yu+etal:2010}. 
For example, clipping the hinge loss leads to the ramp
loss~\cite{collobert06:_tradin,yichaowu07:_robus_trunc_hinge_loss_suppor_vector_machin}. 

In mathematical statistics, robust statistical inference has been studied for a long time. 
A number of robust estimators have been proposed for many kinds of 
statistical problems~\cite{Hampel_etal86,huber81:_robus,maronna06:_robus_statis}. 
In mathematical analysis, one needs to quantify the influence of samples on estimators. 
Here, the influence function, change of variance, and breakdown point are often used as measures of robustness. 
In machine learning literature, these measures are used to analyze the theoretical properties of SVM and its robust
variants. In \cite{christmann04:_robus_proper_convex_risk_minim}, the robustness of a learning
algorithm using a convex loss function was investigated on the basis of an influence function
defined over an RKHS. When the influence function is uniformly bounded 
on the RKHS, the learning algorithm is regarded to be robust against outliers. It was
proved that the quadratic loss function provides a robust learning algorithm for
classification problems in this sense~\cite{christmann04:_robus_proper_convex_risk_minim}. 
From the standpoint of the breakdown point, however, convex loss functions do not provide
robust estimators, as shown in~\cite[Chap.~5.16]{maronna06:_robus_statis}. In
\cite{yu12:_polyn_form_robus_regres,NIPS:Yu+etal:2010}, Yu et al. showed a convex loss
clipping that yields a non-convex loss function and proposed a convex relaxation of the
resulting non-convex optimization problem to obtain a computationally efficient learning
algorithm. They also studied the robustness of the learning algorithm using the clipped loss. 

In this paper, we provide a detailed analysis on the robustness of SVMs. 
In particular, we deal with a robust variant of kernel-based $\nu$-SVM. 
The standard $\nu$-SVM~\cite{NECO:Scholkopf+etal:2000} has a regularization parameter $\nu$, and
it is equivalent with $C$-SVM; i.e., both methods provide the same classifier for the same
training data, if the regularization parameters, $\nu$ and $C$, are properly tuned. 
We also introduce a new robust variant called robust $(\nu,\mu)$-SVM that
has another learning parameter $\mu\in[0,1)$.  
The parameter $\mu$ denotes the ratio of samples to be removed from the training dataset as outliers. 
When the ratio of outliers in the training dataset is bounded above by $\mu$, 
robust $(\nu,\mu)$-SVM is expected to provide a robust classifier. 
Robust $(\nu,\mu)$-SVM is closely related to the robust outlier detection
(ROD) algorithm~\cite{xu06:_robus_suppor_vector_machin_train}. 
Indeed, ROD is to robust $(\nu,\mu)$-SVM what $C$-SVM is to
$\nu$-SVM~\cite{takedaar:_exten_robus_suppor_vector_machin}. 

Our main contribution is to derive the \emph{exact} finite-sample breakdown point of robust $(\nu,\mu)$-SVM. 
The finite-sample breakdown point indicates the largest amount of contamination such that the
estimator still gives information about the non-contaminated data~\cite[Chap.3.2]{maronna06:_robus_statis}. 
We show that the finite-sample breakdown point of robust $(\nu,\mu)$-SVM is equal to $\mu$, 
if $\nu$ and $\mu$ satisfy simple inequalities. 
Conversely, we prove that the finite-sample breakdown point is strictly less than $\mu$, if these key inequalities are
violated. 
The theoretical analysis partly depends on the boundedness of the kernel function used in the statistical model. 
As a result, one can specify the region of the learning parameters $(\nu,\mu)$ such that 
robust $(\nu,\mu)$-SVM has the desired robustness property. 
This property will be of great help to reduce the number of candidate learning parameters $(\nu,\mu)$, 
when the grid search of learning parameters is conducted  with cross validation. 

Some of previous studies are related to ours. In particular, the breakdown point was used to assess the robustness of
kernel-based estimators in \cite{yu12:_polyn_form_robus_regres}. 
In that paper, the influence of a single outlier is considered for a general class of robust estimators. 
In contrast, we focus on a variant of SVM and provide a detailed analysis of the robustness property based on the
breakdown point. In our analysis, an arbitrary number of outliers is taken into account. 

The paper is organized as follows. 
In Section~\ref{sec:Brief_Introduction}, we introduce the problem setup and briefly review the topic of learning
algorithms using the standard $\nu$-SVM. 
Section~\ref{sec:Robust_nu_SVM} is devoted to the robust variant of $\nu$-SVM. 
We show that the dual representation of robust $(\nu,\mu)$-SVM has an intuitive interpretation, that is of great help to
compute the breakdown point. 
An optimization algorithm is also presented. 
In Section~\ref{sec:Breakdown_Point}, we introduce a finite-sample breakdown point as a measure of robustness. 
Then, we evaluate the breakdown point of robust $(\nu,\mu)$-SVM. 
In Section~\ref{sec:Asymptotic_Properties}, we investigate the statistical asymptotic properties of the proposed method
on the basis of order statistics. 
Section~\ref{sec:Numerical_Experiments} examines the generalization performance of robust $(\nu,\mu)$-SVM via numerical
experiments. The conclusion is in Section~\ref{sec:Concluding_Remarks}. 
Detailed proofs of the theoretical results are presented in the Appendix. 

Let us summarize the notations used throughout this paper. Let $\Nbb$ be the set of natural numbers,
and let $[m]$ for $m\in\Nbb$ denote a finite set of $\Nbb$ defined as $\{1,\ldots,m\}$. 
The set of all real numbers is denoted as $\Rbb$. The function $[z]_+$ is defined as $\max\{z,0\}$ 
for $z\in\Rbb$. For a finite set $A$, the size of $A$ is expressed as $|A|$. 
For a reproducing kernel Hilbert space (RKHS) $\mathcal{H}$, 
the norm on $\mathcal{H}$ is denoted as $\|\cdot\|_\mathcal{H}$. 
See \cite{berlinet04:_reprod_hilber} for a description of RKHS.  
Let ${1_m}$ (resp. ${0_m}$) be an $m$-dimensional vector of all ones (resp. all zeros).

\section{Brief Introduction to Learning Algorithms}
\label{sec:Brief_Introduction}
Let us introduce the classification problem with an input space $\mathcal{X}$ and binary output labels $\{+1,-1\}$. 
Given i.i.d. training samples $D=\{(x_i,y_i):i\in[m]\}\subset\mathcal{X}\times\{+1,-1\}$ drawn
from a probability distribution over $\mathcal{X}\times\{+1,-1\}$, a learning algorithm
produces a decision function $g:\mathcal{X}\rightarrow\Rbb$ such that its sign provides 
a prediction of output labels for input points over test samples. 
The decision function $g(x)$ predicts the correct label on the sample $(x,y)$ if and only if the inequality $yg(x)>0$
holds. The product $yg(x)$ is called the margin of the sample $(x,y)$ for the decision
function~$g$~\cite{schapire98:_boost}. To make an accurate decision function, the margins on the training dataset should
take large positive values. 

In kernel-based $\nu$-SVM~\cite{NECO:Scholkopf+etal:2000}, an RKHS $\mathcal{H}$ endowed
with a kernel function $k:\mathcal{X}^2\to\Rbb$ is used to estimate the decision function
$g(x)=f(x)+b$, where $f\in\mathcal{H}$ and $b\in\Rbb$. 
The misclassification penalty is measured by the hinge loss. 
More precisely, $\nu$-SVM produces a decision function $g(x)=f(x)+b$ as the optimal solution of the
convex problem, 
\begin{align}
 \label{eqn:nu-svm}
 \begin{array}{l}
  \displaystyle
 \min_{f,b,\rho}\ \frac{1}{2}\|f\|_{\mathcal{H}}^2
 -\nu\rho+\frac{1}{m}\sum_{i=1}^{m}\big[\rho-y_i(f(x_i)+b\big)]_+\\
  \displaystyle
    \st\ f\in\mathcal{H},\ b,\rho\in\Rbb, 
 \end{array}
\end{align}
where $[\rho-y_i(f(x_i)+b\big)]_+$ is the hinge loss of the margin with the threshold $\rho$. 
The second term $-\nu\rho$ is the penalty for the threshold parameter~$\rho$. 
The parameter $\nu$ in the interval $(0,1)$ is the regularization parameter. 
Usually, the range of $\nu$ that yields a meaningful classifier is narrower than the interval $(0,1)$, as shown in 
\cite{NECO:Scholkopf+etal:2000,DBLP:conf/icml/TakedaS08}. 
The first term in \eqref{eqn:nu-svm} is a regularization term to avoid overfitting to the training data. 
A large positive margin is preferable for each training data. 
The representer theorem~\cite{berlinet04:_reprod_hilber,book:Schoelkopf+Smola:2002}
indicates that the optimal decision function of \eqref{eqn:nu-svm} is of the form, 
\begin{align}
 \label{eqn:f_kernel_representation}
 g(x)=\sum_{j=1}^{m}\alpha_jk(x,x_j)+b
\end{align}
for $\alpha_j\in\Rbb$. The input point $x_j$ with a non-zero coefficient $\alpha_j$ is called a support vector. The
regularization parameter $\nu$ provides a lower bound on the fraction of support vectors. 
Thanks to the representer theorem, even when $\mathcal{H}$ is an infinite dimensional space, the above optimization
problem can be reduced to a finite dimensional quadratic convex problem. 
This is the great advantage of using RKHS for non-parametric statistical 
inference~\cite{NECO:Scholkopf+etal:2000}. 

As pointed out in \cite{DBLP:conf/icml/TakedaS08}, $\nu$-SVM is closely related to a financial risk measure called
conditional value at risk (CVaR)~\cite{JBF:Rockafellar+Uryasev:2002}. 
Roughly speaking, the CVaR of samples $r_1,\ldots,r_m\in\Rbb$ at level $\nu\in(0,1)$ such that
$\nu{m}\in\Nbb$ is defined as the average of its $\nu$-tail, i.e.,
$\frac{1}{\nu{m}}\sum_{i=1}^{\nu{m}}r_{\sigma(i)}$, 
where $\sigma$ is a permutation on $[m]$ such that
$r_{\sigma(1)}\geq\cdots\geq{}r_{\sigma(m)}$ holds. 
In the literature, $r_i$ is defined as the negative margin $r_i=-y_ig(x_i)$. 
For a regularization parameter $\nu$ satisfying $\nu{m}\in\Nbb$ and a fixed decision function
$g(x)=f(x)+b$, the objective function in \eqref{eqn:nu-svm} is expressed as 
 \begin{align*}
&\phantom{=}  \min_{\rho\in\Rbb}
  \frac{1}{2}\|f\|_{\mathcal{H}}^2
  -\nu\rho+\frac{1}{m}\sum_{i=1}^{m}\big[\rho-y_i(f(x_i)+b\big)]_+ \\
&  =
 \frac{1}{2}\|f\|_{\mathcal{H}}^2+\nu\cdot\frac{1}{\nu{m}}\sum_{i=1}^{\nu{m}}r_{\sigma(i)}. 
\end{align*}
Details are presented in Theorem~10 of~\cite{JBF:Rockafellar+Uryasev:2002}. 
Hence, $\nu$-SVM yields a decision function that minimizes the sum of the regularization term and the CVaR of 
the negative margins at level $\nu$. 

In $C$-SVM, the decision function is obtained by solving 
\begin{align}
 \label{eqn:C-svm}
 \begin{array}{l}
  \displaystyle
  \min_{f,b}\ \frac{1}{2}\|f\|_{\mathcal{H}}^2 +C\sum_{i=1}^{m}\big[1-y_i(f(x_i)+b\big)]_+\\
  \displaystyle \st\ f\in\mathcal{H},\ b\in\Rbb. 
 \end{array}
\end{align}
Note that the threshold in the hinge loss is fixed to one in $C$-SVM, whereas $\nu$-SVM determines the threshold with
the optimal solution $\rho$. A positive regularization parameter $C>0$ is used instead of $\nu$. For each training data,
$\nu$-SVM and $C$-SVM can be made to provide the same decision function by appropriately tuning $\nu$ and $C$. 
In this paper, we focus on $\nu$-SVM and its robust variants rather than $C$-SVM. 
The parameter $\nu$ has the explicit meaning shown above, 
and this interpretation will be significant when we derive the robustness property of our method. 

In the robust $C$-SVM proposed
in~\cite{shen03:_learn,yichaowu07:_robus_trunc_hinge_loss_suppor_vector_machin,xu06:_robus_suppor_vector_machin_train}, 
the hinge loss $[1-y_i(f(x_i)+b)]_{+}$ in \eqref{eqn:C-svm} is replaced with the so-called 
ramp loss $\min\{1,\,[1-y_i(f(x_i)+b)]_{+}\}$. By truncating the hinge loss, the influence of outliers is suppressed,
and the estimated classifier is expected to be robust against outliers included in the training data.

\section{Robust $(\nu,\mu)$-SVM}
\label{sec:Robust_nu_SVM}
\subsection{Learning Algorithm of Robust $(\nu,\mu)$-SVM}
\label{subsec:Formulatoin_Robust_SVM}
Here, we propose a robust $(\nu,\mu)$-SVM that is a robust variant of $\nu$-SVM. 
To remove the influence of outliers, we introduce the outlier indicator,
$\eta_i\in\{0,1\},i\in[m]$, for each training sample, where $\eta_i=0$ is intended to indicate
that the sample $(x_i,y_i)$ is an outlier. The same idea is used in~\cite{xu06:_robus_suppor_vector_machin_train}. 
Assume that the ratio of outliers is less than or equal to $\mu$, and define the finite
set 
$E_\mu$ as 
\begin{align*}
 E_\mu=\big\{(\eta_1,\ldots,\eta_m)^T\in\{0,1\}^m\,:\,\sum_{i=1}^m\eta_i\geq{}m(1-\mu)\big\}. 
\end{align*}
For $\nu$ and $\mu$ such that $0<\mu<\nu<1$, robust $(\nu,\mu)$-SVM 
is formalized using RKHS $\mathcal{H}$ as
\begin{align}
\label{eqn:robust_nu_svm} 
 \begin{array}{l}
  \displaystyle
   \min_{f,b,\rho,{\eta}}\ 
   \frac{1}{2}\|f\|_{\mathcal{H}}^2 -(\nu-\mu)\rho
   +\frac{1}{m}\sum_{i=1}^{m}\eta_i\big[\rho-y_i\big(f(x_i)+b\big)\big]_+,  \\
  \displaystyle\ \st\ \ 
   f\in\mathcal{H},\ \ 
   {\eta}=(\eta_1,\ldots,\eta_m)^T\in{}E_\mu,\ \ b,\rho\in\Rbb. 
 \end{array}
\end{align}
The optimal solution, $f\in\mathcal{H}$ and $b\in\Rbb$, provides the decision function
$g(x)=f(x)+b$ for classification. 
Influence from samples with large negative margins is removed by setting $\eta_i$ to zero. 
Throughout the paper, we will assume that $\nu{m}$ and $\mu{m}$ are natural numbers to avoid technical difficulties. 

Robust $(\nu,\mu)$-SVM is closely related to the robust outlier detection (ROD) 
algorithm~\cite{xu06:_robus_suppor_vector_machin_train}. 
About modified algorithms of ROD and robust $(\nu,\mu)$-SVM, 
the equivalence is shown in~\cite{takedaar:_exten_robus_suppor_vector_machin}. 
In ROD, the classifier is given by the optimal solution of 
\begin{align}
 \label{eqn:ROD}
 \begin{array}{l}
  \displaystyle
   \min_{f,b,\eta}\frac{\lambda}{2}\|f\|_{\mathcal{H}}^2+\sum_{i=1}^{m}\eta_i[1-y_i(f(x_i)+b)]_{+}, \\
  \displaystyle \st\  f\in\mathcal{H},\ \ b\in\Rbb,\\
  \phantom{\st}\ {\eta}=(\eta_1,\ldots,\eta_m)^T\in{}[0,1]^m,\,\
   \sum_{i=1}^{m}\eta_i\geq{}m(1-\mu), 
 \end{array}
\end{align}
where $\lambda>0$ is a regularization parameter. In the original ROD, the linear kernel is used. 
To obtain the classifier, the ROD algorithm solves a semidefinite relaxation of the above problem. 

Furthermore, robust $(\nu,\mu)$-SVM is related to CVaR at levels $\nu$ and $\mu$. 
Indeed, for the parameters, $\nu$ and $\mu$, and a fixed decision function $g(x)=f(x)+b$, 
the objective function in \eqref{eqn:robust_nu_svm} is represented as 
\begin{align}
&\phantom{=}  
 \min_{\rho\in\Rbb,{\eta}\in{E_\mu}}\ 
   \frac{1}{2}\|f\|_{\mathcal{H}}^2 -(\nu-\mu)\rho
 +\frac{1}{m}\sum_{i=1}^{m}\eta_i\big[\rho+r_i\big]_+ \nonumber\\
&=
 \min_{\rho\in\Rbb}\ 
 \frac{1}{2}\|f\|_{\mathcal{H}}^2 -\nu\rho
 +\frac{1}{m}\sum_{i=1}^{m}\big[\rho+r_i\big]_+ 
 -\max_{{\eta}\in{E_\mu}}\frac{1}{m}\sum_{i=1}^m(1-\eta_i)r_i
 \label{eqn:DC-representation} \\
&=
 \frac{1}{2}\|f\|_{\mathcal{H}}^2 
 +(\nu-\mu)\cdot\frac{1}{(\nu-\mu)m}\sum_{i=\mu{m}+1}^{\nu{m}}\!\!r_{\sigma(i)},
 \label{eqn:diff_CVaR}
\end{align}
where $r_i=-y_i(f(x_i)+b)$ is the negative margin and $r_{\sigma(i)}$ is its sort in the
descending order defined in Section~\ref{sec:Brief_Introduction}. 
The second term in \eqref{eqn:diff_CVaR} is the average of the negative margins included
in the middle interval presented in Figure~\ref{fig:margin_hist}, 
and it is expressed by the difference of CVaRs at levels $\nu$ and $\mu$. 
The learning algorithm based on this interpretation is proposed
in~\cite{tsyurmasto13:_suppor_vector_class_posit_homog_risk_funct} under the name CVaR-$(\alpha_L,\alpha_U)$-SVM.  
The two methods can be shown to be equivalent by setting $\alpha_L=1-\nu$ and $\alpha_U=1-\mu$. 
In this paper, the learning algorithm based on \eqref{eqn:robust_nu_svm} is referred to as robust $(\nu,\mu)$-SVM to
emphasize that it is a robust variant of $\nu$-SVM.  

 \begin{figure}[t]
 \begin{center}
  \includegraphics[scale=0.4]{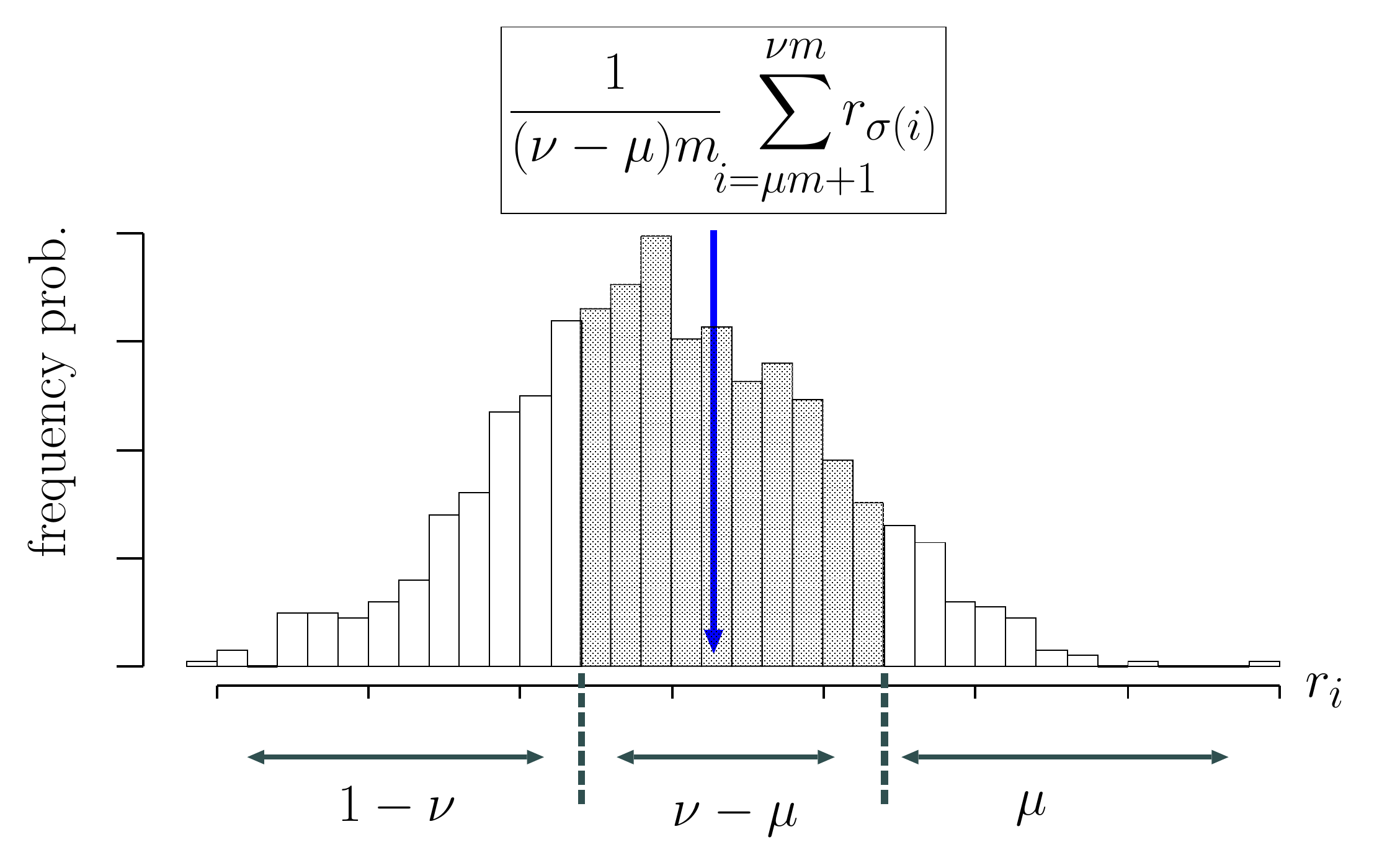}
 \end{center}
 \caption{Distribution of negative margins $r_i=-y_ig(x_i),i\in[m]$ for a fixed decision function $g(x)=f(x)+b$.} 
 \label{fig:margin_hist}
\end{figure}

The representer theorem ensures that the optimal decision function of \eqref{eqn:robust_nu_svm} 
is represented by $f(x)=\sum_{i=1}^{m}\alpha_ik(x,x_i)+b$ when the kernel function of the RKHS
$\mathcal{H}$ is given by $k(x,x')$. 
As in the case of the standard $\nu$-SVM, the number of support vectors, i.e., the input points $x_i$ such
that $\alpha_i\neq0$, is bounded below by $(\nu-\mu)m$. 
In addition, the KKT condition of \eqref{eqn:robust_nu_svm} leads to the fact that 
any support vector $x_i$ satisfies $\eta_i=1$. 

It is hard to obtain a global optimal solution of \eqref{eqn:robust_nu_svm}, since the objective
function is non-convex. 
As shown in~
\cite{tsyurmasto13:_suppor_vector_class_posit_homog_risk_funct,takedaar:_exten_robus_suppor_vector_machin}, 
the objective function in \eqref{eqn:robust_nu_svm} is
expressed as a difference of convex functions~(DC) by using a CVaR representation. 
Hence, the DC algorithm~\cite{Pham:1997} and convex-concave programming 
(CCCP)~\cite{NECO:Yuille+Rangarajan:2003} are available to efficiently obtain 
a stationary point of \eqref{eqn:robust_nu_svm}.  
The same approach is taken by robust $C$-SVM using the ramp loss~\cite{collobert06:_tradin}. 

In Algorithm~\ref{alg:DC_alg_robust_numuSVM}, the DC algorithm for robust $(\nu,\mu)$-SVM based on the expression
\eqref{eqn:DC-representation} is presented. 
The derivation of the DC algorithm is presented in Appendix~\ref{appendix:DCA}. 
Algorithm~\ref{alg:DC_alg_robust_numuSVM} is guaranteed to converge in a finite number of iterations. 
In the DC algorithm, a monotone decrease of the objective value is generally
guaranteed, and in Algorithm~\ref{alg:DC_alg_robust_numuSVM}, the objective value in each 
iteration is determined by ${\eta}\in{E_\mu}$, which can take only a finite number of distinct values. 
A stationary point is obtained when the objective value is unchanged. 
The above argument is based on a convergence analysis of robust $C$-SVM using the ramp loss~\cite{collobert06:_tradin}. 
In addition, an argument based on polyhedral DC programming shows that the algorithm converges after a finite number of
iterations~\cite{an05:_dc_differ_convex_funct_progr}. 
One can use another stopping rule such that the algorithm terminates when the same ${\eta}\in{E_\mu}$
is obtained in two consecutive iterations. 
If the cyclic phenomenon of ${\eta}$ is prohibited in some way, 
convergence in a finite number of iterations is guaranteed. 

\begin{algorithm}[t]
 \caption{DC algorithm for robust $(\nu,\mu)$-SVM}
 \label{alg:DC_alg_robust_numuSVM}                          
\begin{algorithmic}[1]
 \REQUIRE Gram matrix $K\in\Rbb^{m\times{m}}$ defined as $K_{ij}=k(x_i,x_j),i,j\in[m]$, 
 and
 training labels $y=(y_1,\ldots,y_m)^T\in\{+1,-1\}^m$. 
 The matrix $\widetilde{K}\in\Rbb^{m\times{m}}$ is defined as $\widetilde{K}_{ij}=y_iy_jK_{ij}$. 
 Let $g(x)=f(x)+b$ be an initial decision function. 
 \REPEAT
 \STATE 
 Compute the sort $r_{\sigma(1)}\geq\cdots\geq{r_{\sigma(m)}}$ of the negative margin $r_i=-y_ig(x_i)$, and 
 set 
 \begin{align*}
  \eta_{\sigma(i)}\leftarrow\begin{cases}
	   0, &  1\leq{i}\leq\mu{m},\\
	   1, &  \text{otherwise},
	  \end{cases}
 \end{align*}
 for $i\in[m]$. Let ${\eta}$ be $(\eta_1,\ldots,\eta_m)^T\in{E_\mu}$. 
 \STATE 
 Set $c\leftarrow-\widetilde{K}(1_m-{\eta})/m$ and 
 $d\leftarrow{{y}}^T(1_m-{\eta})/m$.  
 \STATE Compute the optimal solution ${\beta}_{\mathrm{opt}}$ of the problem
 \begin{align}
  \label{eqn:DC_alg_quadprob}
  \min_{{\beta}\in\Rbb^m}\frac{1}{2}{\beta}^T\widetilde{K}{\beta}+c^T{\beta},\quad
  \st {0}_m\leq{\beta}\leq{1}_m/m,\ \ {\beta}^T{ y}=d,\ \ {\beta}^T{1_m}=\nu. 
 \end{align}
 \STATE Set ${\alpha}\leftarrow{{y}}\circ({\beta}_{\mathrm{opt}}-(1_m-{\eta})/m)$, 
 where $\circ$ denotes component-wise multiplication of two vectors. 
 \STATE Compute $\rho$ and $b$ using 
 $0<\beta_i<1/m\ \Longrightarrow\ \rho=y_ig(x_i)$, where $g(x_i)=\sum_{j=1}^{m}K_{ij}\alpha_j+b$. 
 \UNTIL{the objective value of \eqref{eqn:robust_nu_svm} is unchanged.}
\RETURN the decision function $\displaystyle{}g(x)=\sum_{i=1}^m{}k(x,x_i)\alpha_i+b$. 
\end{algorithmic}
\end{algorithm}

\subsection{Dual Problem and Its Interpretation}
\label{subsec:Dual_Problem}

The partial dual problem of \eqref{eqn:robust_nu_svm} with a fixed outlier indicator
${\eta}=(\eta_1,\ldots,\eta_m)\in{E_\mu}$ has an intuitive geometric picture.  
Some variants of $\nu$-SVM can be geometrically interpreted on the basis of the dual
form~\cite{CriBur00,kanamori13:_conjug_relat_loss_funct_uncer,ICML2012Takeda_87}. 
Substituting \eqref{eqn:f_kernel_representation} into the objective function in 
\eqref{eqn:robust_nu_svm}, we obtain the Lagrangian of problem~\eqref{eqn:robust_nu_svm} with 
a fixed $\eta\in{E_\mu}$ as  
\begin{align*}
L_{\eta}(\alpha,b,\rho,\xi;\beta,\gamma)
&=
 \frac{1}{2}\sum_{i,j=1}^m\alpha_i\alpha_jk(x_i,x_j)
 -(\nu-\mu)\rho+\frac{1}{m}\sum_{i=1}^m\eta_i\xi_i
 -\sum_{i=1}^m\beta_i\xi_i\\
 &\phantom{=}
 +\sum_{i=1}^m\gamma_i\bigg(\rho-\xi_i-y_i\bigg(\sum_{j}k(x_i,x_j)\alpha_j+b\bigg)\bigg), 
\end{align*}
where non-negative slack variables $\xi_i,\,i\in[m]$ are introduced to 
represent the hinge loss. Here, the parameters $\beta_i$ and $\gamma_i$ for $i\in[m]$ are non-negative Lagrange 
multipliers.  
For a fixed ${\eta}\in{E_\mu}$, the Lagrangian is convex in the parameters
$\alpha,b,\rho$, and $\xi$ and concave in $\beta=(\beta_1,\ldots,\beta_m)$ and $\gamma=(\gamma_1,\ldots,\gamma_m)$. 
Hence, the min-max theorem~\cite[Proposition~6.4.3]{book:Bertsekas+etal:2003} yields 
\begin{align*}
 &\phantom{=}
 \inf_{\alpha,b,\rho,\xi} \sup_{\beta,\gamma\geq0} L_\eta(\alpha,b,\rho,\xi;\beta,\gamma)\\
&=
 \sup_{\beta,\gamma\geq0}\inf_{\alpha,b,\rho,\xi} L_\eta(\alpha,b,\rho,\xi;\beta,\gamma)\\ 
&=
  \sup_{\beta,\gamma\geq0}\inf_{\alpha,b,\rho,\xi}
  \rho\bigg(\sum_{i}\gamma_i-(\nu-\mu)\bigg)
 +\sum_{i}\xi_i\bigg(\frac{\eta_i}{m}-\beta_i-\gamma_i\bigg)\\
 &\phantom{=}\qquad 
 +\frac{1}{2}\sum_{i,j}\alpha_i\alpha_jk(x_i,x_j)-\sum_{i}\gamma_iy_i\sum_{j}k(x_i,x_j)\alpha_j -b\sum_{i}y_i\gamma_i \\
&=
 \max
 \bigg\{
 -\frac{1}{2}
\bigg\| \sum_{i}\gamma_iy_ik(\cdot,x_i)\bigg\|_\mathcal{H}^2
 \ :\ \sum_{i:y_i=+1}\gamma_i=\sum_{i:y_i=-1}\gamma_i=\frac{\nu-\mu}{2},\ 
 0\leq\gamma_i\leq\frac{\eta_i}{m}
 \bigg\}. 
\end{align*}

Let us give a geometric interpretation of the above expression. 
For the training data $D=\{(x_i,y_i):i\in[m]\}$, the convex sets, 
$\mathcal{U}_{\eta}^{+}[\nu,\mu;D]$ 
and 
$\mathcal{U}_{\eta}^{-}[\nu,\mu;D]$, are defined as the reduced convex hulls of data points for each label, i.e., 
\begin{align*}
 &\phantom{=}\mathcal{U}_{\eta}^{\pm}[\nu,\mu;D]\\
 &=
 \bigg\{ 
 \sum_{i:y_i=\pm1}\gamma_i'k(\cdot,x_i)\in\mathcal{H}:\!\!\!
 \sum_{i:y_i=\pm1}\gamma_i'=1,\  0\leq \gamma_i'\leq \frac{2\eta_i}{(\nu-\mu)m} \
 \text{for $i$ such that $y_i=\pm1$} 
 \bigg\}. 
\end{align*}
The coefficients $\gamma_i',\,i\in[m]$ in $\mathcal{U}_{\eta}^{\pm}[\nu,\mu;D]$ are bounded above by
a non-negative real number that is usually less than one. Hence, the reduced convex hull is a subset of the convex hull
of the data points in the RKHS $\mathcal{H}$. 
Each reduced convex hull is regarded as the domain of the input samples of each label. 
Accordingly, let $\mathcal{V}_{\eta}[\nu,\mu;D]$ be the Minkowski difference of two subsets, 
\begin{align*}
 \mathcal{V}_{\eta}[\nu,\mu;D]= \mathcal{U}_{\eta}^{+}[\nu,\mu;D] \ominus \mathcal{U}_{\eta}^{-}[\nu,\mu;D], 
\end{align*}
where $A\ominus{B}$ of subsets $A$ and $B$ denotes $\{a-b:a\in{A},\,b\in{B}\}$. 
Eventually, for each ${\eta}\in{E_\mu}$, the optimal value in the above is represented by 
\begin{align*}
 \inf_{\alpha,b,\rho,\xi} \sup_{\beta,\gamma\geq0} L_\eta(\alpha,b,\rho,\xi;\beta,\gamma)
 \,=\,
 -\frac{(\nu-\mu)^2}{8}\min\left\{\|f\|_{\mathcal{H}}^2\,:\,f\in\mathcal{V}_\eta[\nu,\mu;D]\right\}. 
\end{align*}
Hence, the optimal value of \eqref{eqn:robust_nu_svm} is $-(\nu-\mu)^2/8\times\mathrm{opt}(\nu,\mu;D)$, where 
\begin{align}
\label{eqn:dual_robust_nuSVM}
\mathrm{opt}(\nu,\mu;D)
=
\max_{\eta\in{E_\mu}}\min_{f\in\mathcal{V}_\eta[\nu,\mu;D]}\|f\|_{\mathcal{H}}^2. 
\end{align}

Therefore, the dual form of robust $(\nu,\mu)$-SVM is expressed as the maximization of
the minimum distance between two reduced convex hulls, 
$\mathcal{U}_{\eta}^{+}[\nu,\mu;D]$ and $\mathcal{U}_{\eta}^{-}[\nu,\mu;D]$. 
The estimated decision function in robust $(\nu,\mu)$-SVM is provided by the optimal
solution of \eqref{eqn:dual_robust_nuSVM} up to a scaling factor depending on $\nu-\mu$. 
Moreover, the optimal value is proportional to the squared RKHS norm of the function
$f(x)\in\mathcal{H}$ in the decision function $g(x)=f(x)+b$.

\section{Breakdown Point of Robust $(\nu,\mu)$-SVM}
\label{sec:Breakdown_Point}

\subsection{Finite-Sample Breakdown Point}
\label{subsec:finite-sample_breakdown_point}
Let us describe how to evaluate the robustness of learning algorithms. There are a number of robustness measures 
for evaluating the stability of estimators. 
For example, the influence function evaluates the infinitesimal bias of the estimator caused by a few outliers included
in the training samples. 
The gross error sensitivity is the worst-case infinitesimal bias defined with the influence function~\cite{maronna06:_robus_statis}.  
In this paper, we use the \emph{finite-sample breakdown point}, and it will be referred to as the breakdown point for short. 
The breakdown point quantifies the degree of impact that the outliers have on the estimators when the contamination
ratio is not necessarily infinitesimal~\cite{donoho83}. 
In this section, we present an exact evaluation of the breakdown point of robust $(\nu,\mu)$-SVM. 

The breakdown point indicates the largest amount of contamination such that the estimator still gives information about
the non-contaminated data~\cite[Chap.3.2]{maronna06:_robus_statis}. 
More precisely, for an estimator $\theta_D$ based on a dataset $D$ of size $m$ that
takes a value in a normed space, 
the finite-sample breakdown point is defined as 
\begin{align*}
 \varepsilon^*=\max_{\kappa=0,1,\ldots,m}\{\, \kappa/m\,:\, \text{$\theta_{D'}$ is
 uniformly bounded 
 for $D'\in\mathcal{D}_{\kappa}$\,}\,\},  
\end{align*}
where 
$\mathcal{D}_\kappa$ is the family of datasets of size $m$ including at least $m-\kappa$ elements in common with the 
non-contaminated dataset $D$, i.e.,  
\begin{align*}
 \mathcal{D}_{\kappa}=
 \big\{\,D'=\{(x_i',y_i'):i\in[m]\}\subset\mathcal{X}\times\{+1,-1\}\,:\,|D'\cap{D}|\geq{m-\kappa}\,\big\}. 
\end{align*}
For simplicity, the dependency of $\mathcal{D}_{\kappa}$ on the data set $D$ is dropped. 
The condition of the breakdown point $\varepsilon^*$ can be rephrased as
\begin{align*}
\sup_{D'\in\mathcal{D}_{\kappa}}\|\theta_{D'}\|<\infty, 
\end{align*}
where $\|\cdot\|$ is the norm on the normed space. 
In most cases of interest, $\varepsilon^*$ does not 
depend on the dataset $D$. For example, the breakdown point of the one-dimensional median estimator is
$\varepsilon^*=\lfloor(m-1)/2\rfloor/m$. 

To start with, let us derive a lower bound of the breakdown point for the optimal value of
problem~\eqref{eqn:robust_nu_svm} that is expressed as $\mathrm{opt}(\nu,\mu;D)$ up to a constant factor.  
As shown in Section~\ref{subsec:Dual_Problem}, the boundedness of $\mathrm{opt}(\nu,\mu;D)$ is equivalent to the
boundedness of the RKHS norm of $f\in\mathcal{H}$ in the estimated decision function $g(x)=f(x)+b$.  
Given a labeled dataset $D=\{(x_i,y_i):i\in[m]\}$, let us define the label ratio $r$ as 
\begin{align*}
r=\frac{1}{m}\min\{\,|\{i:y_i=+1\}|,|\{i:y_i=-1\}|\,\}.  
\end{align*}

\begin{theorem}
 \label{theorem:breakdown-point-optvalue}
 Let $D$ be a labeled dataset of size $m$ with a positive label ratio $r$. 
 For the parameters $\nu,\mu$ such that $0\leq\mu<\nu<1$ and $\nu{m},\mu{m}\in\Nbb$, 
 we assume $\mu<r/2$. 
 Then, the following two conditions are equivalent. 
 \begin{description}
  \item[{\rm (i)} ] The inequality 
	     \begin{align}
	      \label{eqn:key_inequality}
	      \nu-\mu\leq2(r-2\mu)
	     \end{align}
	     holds. 
  \item[{\rm (ii)}] Uniform boundedness, 
	     \begin{align*}
	      \sup\{\mathrm{opt}(\nu,\mu;D')\,:\,D'\in\mathcal{D}_{\mu{m}}\} <\infty
	     \end{align*}
	     holds, where $\mathcal{D}_{\mu{m}}$ is the family of contaminated datasets defined from~$D$. 
 \end{description}
\end{theorem}

The proof is given in Appendix~\ref{appendix:proof_breakdown_obj}. The inequality $\mu<r/2$ is a requisite condition. 
If this inequality is violated, the majority of, say, positive labeled samples in the non-contaminated training dataset
can be replaced with outliers. In such a situation, the statistical features in the original dataset will not be retained. 
Indeed, if $\mu\geq{}r/2$ holds, $\mathrm{opt}(\nu,\mu;D')$ is unbounded over $D'\in\mathcal{D}_{\mu{m}}$ regardless of
$\nu$. Since it is proved by a rigorous description of the above intuitive interpretation, the proof is omitted. 
Theorem~\ref{theorem:breakdown-point-optvalue} indicates that the breakdown point of the RKHS element in the estimated
decision function is greater than or equal to $\mu$, if $\mu$ and $\nu$ satisfy inequality~\eqref{eqn:key_inequality}. 
Conversely, if the inequality $\nu-\mu\leq2(r-2\mu)$ is violated, 
the breakdown point of robust $(\nu,\mu)$-SVM does not reach $\mu$, even though $\mu{m}$ samples are
removed from the training data. 
In addition, the inequality \eqref{eqn:key_inequality} indicates the trade-off between the ratio of outliers $\mu$ and
the ratio of support vectors $\nu-\mu$. This result is reasonable. The number of support vectors corresponds 
to the dimension of the statistical model. When the ratio of outliers is large, a simple statistical model should be
used to obtain robust estimators. 
If there is no outlier in training data, i.e., $\mu=0$, 
inequality \eqref{eqn:key_inequality} reduces to $\nu\leq{2r}$. 
For the standard $\nu$-SVM, this is a necessary and sufficient condition for the optimization problem 
\eqref{eqn:nu-svm} to be bounded~\cite{CriBur00}. 

When the contamination ratio in a training dataset is greater than $\mu$, 
the estimated decision function is not necessarily bounded. 
\begin{theorem}
 \label{theorem:breakdown-point_upper_bound}
 Suppose that $\nu$ and $\mu$ are rational numbers such that $0<\mu<1/4$ and $\mu<\nu<1$. 
 Then, there exists a dataset $D$ of size $m$ with the label ratio $r$ such that 
 $\mu<r/2$ and 
 \begin{align*}
  \sup\{\mathrm{opt}(\nu,\mu;D')\,:\,D'\in\mathcal{D}_{\mu{m}+1}\}=\infty
 \end{align*}
 hold, where $\mathcal{D}_{\mu{m}+1}$ is defined from $D$. 
 \end{theorem}
The proof is given in Appendix~\ref{appendix:upperbound_breakdown_opt}. 
Theorems~\ref{theorem:breakdown-point-optvalue} and \ref{theorem:breakdown-point_upper_bound}
lead to the fact that the breakdown point of the function part $f\in\mathcal{H}$ in the estimated decision function
$g=f+b$ is exactly equal to $\varepsilon^*=\mu$, when the learning parameters of the robust $(\nu,\mu)$-SVM satisfy
$\mu<r/2$ and $\nu-\mu\leq2(r-2\mu)$. Otherwise, the breakdown point of $f$ is strictly less than $\mu$. 

We show the robustness of the bias term $b$. 
Let $b_D$ be the estimated bias parameter obtained by robust $(\nu,\mu)$-SVM from the training dataset~$D$. 
We will derive a lower bound of the breakdown point of the bias term. 
Then, we will show that the breakdown point of robust $(\nu,\mu)$-SVM with a bounded kernel is given 
by a simple formula. 
\begin{theorem}
 \label{theorem:bounded_bias}
 Let $D$ be an arbitrary  dataset of size $m$ with a positive label ratio $r$. 
 Suppose that $\nu$ and $\mu$ satisfy 
 $0<\mu<\nu<1$, $\nu{m},\mu{m}\in\Nbb$, and $\mu<r/2$. 
 For a non-negative integer $\ell$, we assume 
 \begin{align}
  \label{eqn:breakdown_point_bias}
  0\leq2\left(\mu-\frac{\ell}{m}\right)<\nu-\mu<2(r-2\mu). 
 \end{align}
 Then, uniform boundedness  
 \begin{align*}
  \sup\{\,|b_{D'}|:D'\in\mathcal{D}_{\mu{m}-\ell}\,\}<\infty
 \end{align*}
 holds, where $\mathcal{D}_{\mu{m}-\ell}$ is defined from $D$. 
\end{theorem} 
The proof is given in Appendix~\ref{appendix:proof_breakdown_bias}. Note that the
inequality~\eqref{eqn:breakdown_point_bias} is a sufficient condition of inequality~\eqref{eqn:key_inequality}. 
Theorem~\ref{theorem:bounded_bias} guarantees that the breakdown point of the estimated decision function $f+b$ is not
less than $\mu-\ell/m$ when \eqref{eqn:breakdown_point_bias} holds. 

When the kernel function is bounded, the boundedness of the function part $f\in\mathcal{H}$ in the decision
function $f+b$ almost guarantees the boundedness of the bias term $b$. 
\begin{theorem}
 \label{theorem:bounded_kernel_bounded_bias}
 Let $D$ be an arbitrary dataset of size $m$ with a positive label ratio $r$. 
 For the parameters $\nu,\mu$ such that 
 $0<\mu<\nu<1$ and $\nu{m},\mu{m}\in\Nbb$, suppose that $\mu<r/2$ and $\nu-\mu<2(r-2\mu)$ hold. 
 In addition, assume that the kernel function $k(x,x')$ of the RKHS $\mathcal{H}$ is bounded, i.e.,
 $\sup_{x\in\mathcal{X}}k(x,x)<\infty$. Then, uniform boundedness, 
 \begin{align*}
  \sup\{\,|b_{D'}|:D'\in\mathcal{D}_{\mu{m}}\,\}<\infty, 
 \end{align*}
 holds, where $\mathcal{D}_{\mu{m}}$ is defined from $D$. 
\end{theorem}
The proof is given in Appendix~\ref{appendix:proof_bounded_kernel_breakdown_bias}. 
Compared with Theorem~\ref{theorem:bounded_bias} in which arbitrary kernel functions are treated,
Theorem~\ref{theorem:bounded_kernel_bounded_bias} ensures that a tighter lower bound of the breakdown point is
obtained for bounded kernels.  The above result agrees with those of other studies. 
The authors of \cite{yu12:_polyn_form_robus_regres} proved that bounded kernels produce
robust estimators for regression problems in the sense of bounded response, i.e.,
robustness against a single outlier. 

\begin{figure}[t]
  \begin{tabular}{cc}
   Bounded kernel  &  Unbounded kernel  \\
   \includegraphics[scale=0.35]{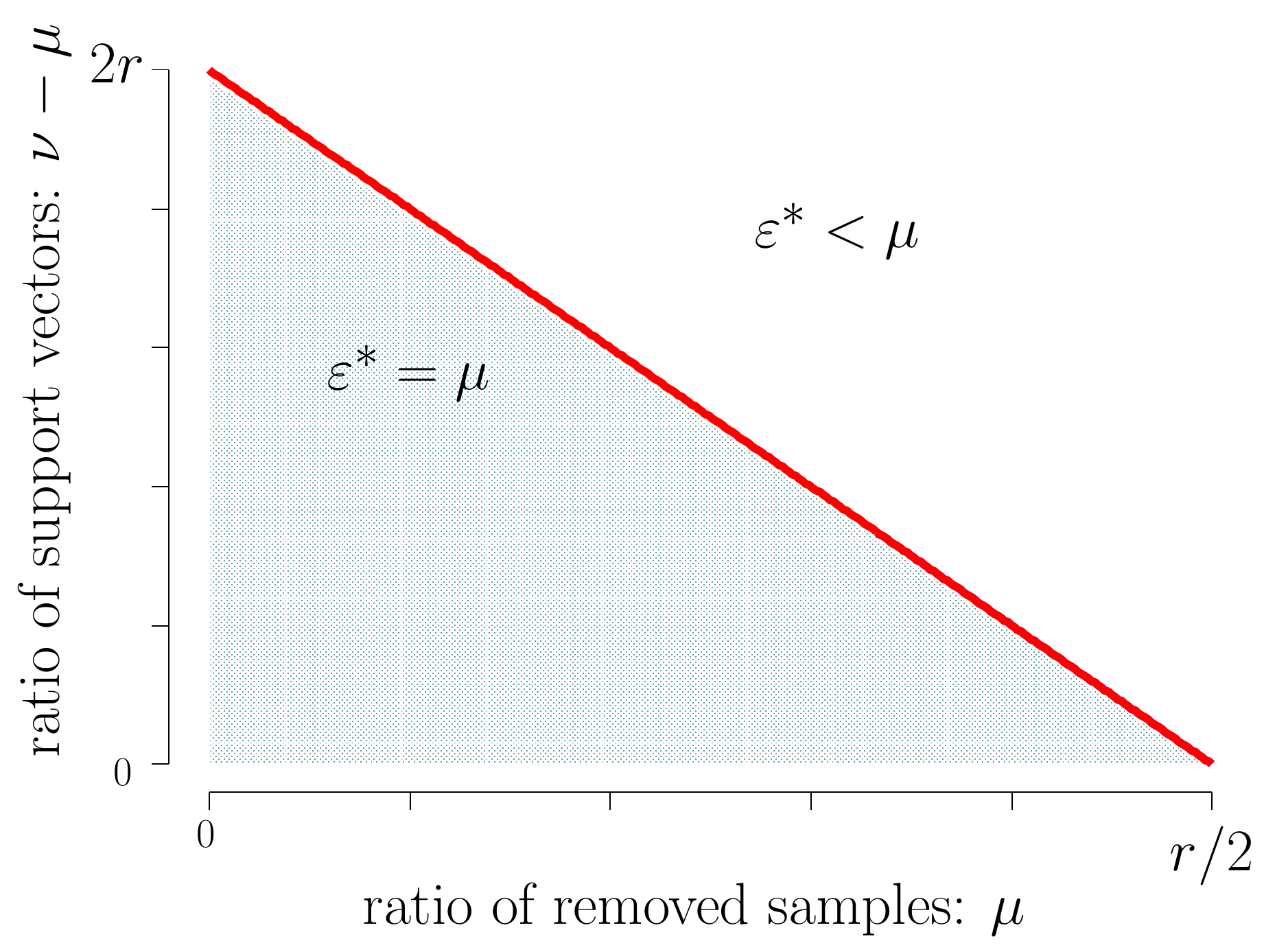} &
   \includegraphics[scale=0.35]{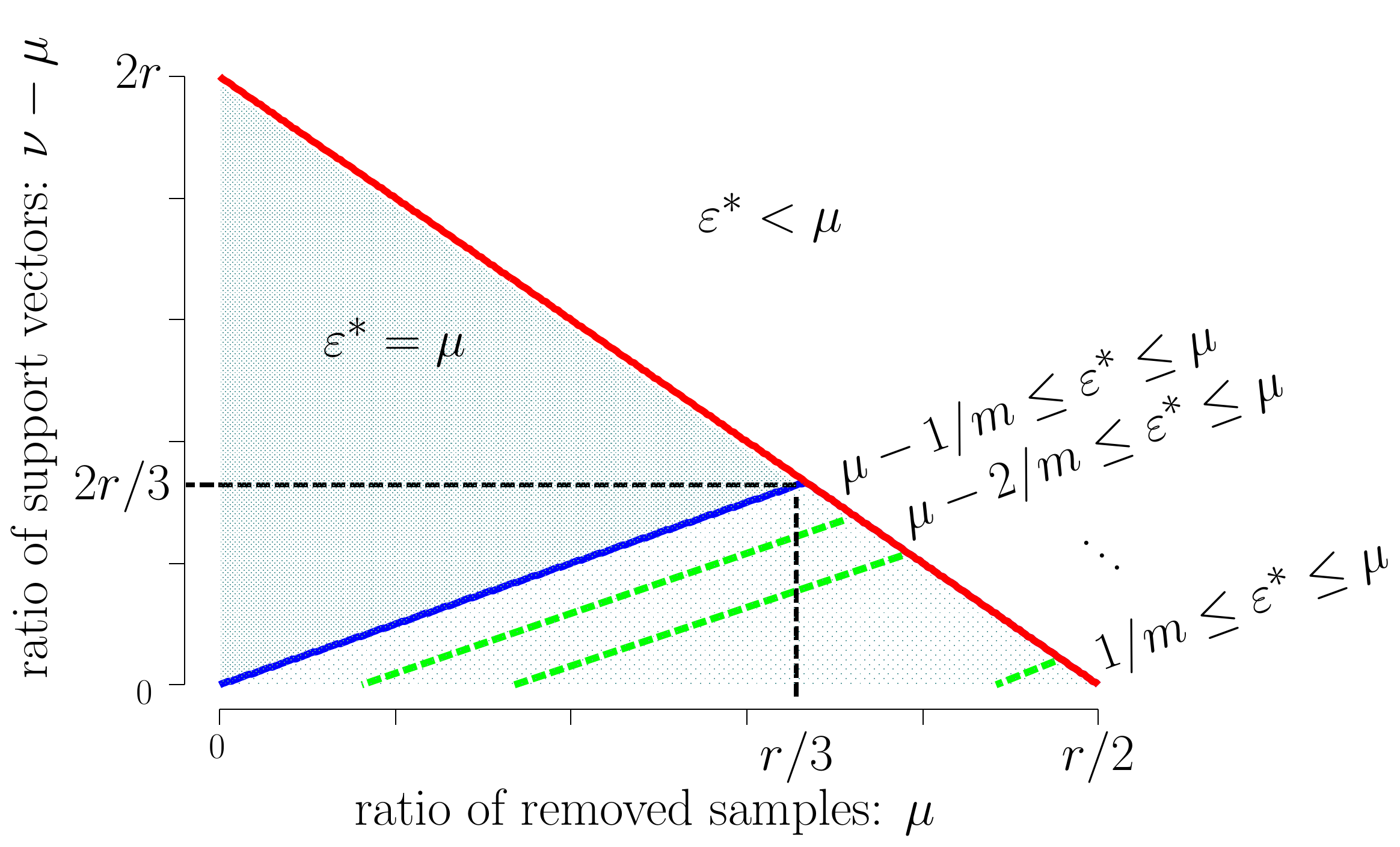} 
  \end{tabular}
 \caption{Left (resp. Right) panel: breakdown point of $(f,b)\in\mathcal{H}\times\Rbb$ given by robust $(\nu,\mu)$-SVM
 with bounded (resp. unbounded) kernel.}
 \label{fig:breakdown-ponit}
\end{figure}

Combining Theorems~\ref{theorem:breakdown-point-optvalue}, 
\ref{theorem:breakdown-point_upper_bound}, \ref{theorem:bounded_bias} and
\ref{theorem:bounded_kernel_bounded_bias}, we find that the breakdown point 
of $(\nu,\mu)$-SVM with $\mu<r/2$ is given as follows. 
\begin{description}
 \item[Bounded kernel:]
	    For $\nu-\mu>2(r-2\mu)$, the breakdown point of $f\in\mathcal{H}$ is less than $\mu$. 
	    For $\nu-\mu\leq2(r-2\mu)$, the breakdown point of $(f,b)\in\mathcal{H}\times\Rbb$ is equal to $\mu$. 
 \item[Unbounded kernel:]
	    For $\nu-\mu>2(r-2\mu)$, the breakdown point of $f\in\mathcal{H}$ is less than $\mu$. 
	    For $2\mu<\nu-\mu\leq2(r-2\mu)$, the breakdown point of $(f,b)\in\mathcal{H}\times\Rbb$ is equal to $\mu$. 
	    When $0<\nu-\mu<\min\{2\mu,2(r-2\mu)\}$, the breakdown point of the function part $f$ is equal to $\mu$, and 
	    the breakdown point of the bias term $b$ is bounded from below by $\mu-\ell/m$ and from above by $\mu$,
	    where $\ell\in\Nbb$ depends on $\nu$ and $\mu$, as shown in Theorem~\ref{theorem:bounded_bias}. 
\end{description}
Figure~\ref{fig:breakdown-ponit} shows the breakdown point of robust $(\nu,\mu)$-SVM. 
The line $\nu-\mu=2(r-2\mu)$ is critical. For unbounded kernels, we obtain only a bound of the breakdown
point. Hence, there is a possibility that unbounded kernels provide the same breakdown point as bounded
kernels.

\subsection{Acceptable Region for Learning Parameters}
\label{subsec:Acceptable_Region_Learning_Parameters}

The theoretical analysis in Section~\ref{subsec:finite-sample_breakdown_point} 
suggests that robust $(\nu,\mu)$-SVM satisfying
$0<\nu-\mu<2(r-2\mu)$ is a good choice for obtaining a robust classifier, especially when a bounded kernel is
used. Here, $r$ is the label ratio of the non-contaminated original data $D$, and usually it is unknown in real-world
data analysis. Thus, we need to estimate $r$ from the contaminated dataset $D'$. 

If an upper bound of the outlier ratio is known to be 
$\bar{\mu}$, we have $D'\in\mathcal{D}_{\bar{\mu}m}$, where $\mathcal{D}_{\bar{\mu}m}$ is defined
from $D$. Let $r'$ be the label ratio of $D'$. Then, the label ratio of the original dataset $D$
should satisfy $r_{\mathrm{low}}\leq{r}\leq{}r_{\mathrm{up}}$, where
$r_{\mathrm{low}}=\max\{r'-\bar{\mu},0\}$ and $r_{\mathrm{up}}=\min\{r'+\bar{\mu},1/2\}$. 
Let $\Lambda_{\mathrm{low}}$ and $\Lambda_{\mathrm{up}}$ be
\begin{align*}
 \begin{array}{l}
  \displaystyle
  \Lambda_{\mathrm{low}}=\{(\nu,\mu)\,:\,0\leq\mu\leq\bar{\mu},\,0<\nu-\mu<2(r_{\mathrm{low}}-2\mu)\},\\
  \displaystyle
   \ \Lambda_{\mathrm{up}}=\{(\nu,\mu)\,:\,0\leq\mu\leq\bar{\mu},\,0<\nu-\mu<2(r_{\mathrm{up}}-2\mu)\}. 
 \end{array}
\end{align*}
Then, robust $(\nu,\mu)$-SVM with $(\nu,\mu)\in\Lambda_{\mathrm{low}}$ reaches the breakdown point $\mu$ for any
non-contaminated dataset $D$ such that $D'\in\mathcal{D}_{\mu{m}}$ for given $D'$. 
On the other hand, the parameters $(\nu,\mu)$ on the outside of $\Lambda_{\mathrm{up}}$ is not necessary. Indeed, the
parameter $\mu$ such that $0<\mu\leq\bar{\mu}$ is sufficient to detect outliers. In addition, for any non-contaminated
data $D$ such that $D'\in\mathcal{D}_{\bar{\mu}m}$ for given $D'$, $(\nu,\mu)$ satisfying
$\nu-\mu>2(r_{\mathrm{up}}-2\mu)$ does not yield a learning method that reaches the breakdown point $\mu$. 

When an upper bound $\bar{\mu}$ is unknown, we use $\bar{\mu}=r/2$ and obtain
$\bar{r}_{\mathrm{low}}\leq{r}\leq{}\bar{r}_{\mathrm{up}}$, where
$\bar{r}_{\mathrm{low}}=2r'/3$ and $\bar{r}_{\mathrm{up}}=\min\{2r',1/2\}$. 
Hence, in the worst case, the admissible set of the learning parameters $\nu$ and $\mu$ is
given as 
\begin{align}
 \label{eqn:worst-case_par_region} 
 \begin{array}{l}
  \displaystyle
   \overline{\Lambda}_{\mathrm{low}}=\{(\nu,\mu)\,:\,0<\nu-\mu<2(\bar{r}_{\mathrm{low}}-2\mu)\},\\
  \displaystyle  \ 
   \overline{\Lambda}_{\mathrm{up}}=\{(\nu,\mu)\,:\,0<\nu-\mu<2(\bar{r}_{\mathrm{up}}-2\mu)\}. 
 \end{array}
\end{align}
Given contaminated training data $D'$, for any $D$ of size $m$ with a label ratio
$r\in[\bar{r}_{\mathrm{low}},\bar{r}_{\mathrm{up}}]$ such that $D'\in\mathcal{D}_{\mu{m}}$ with 
$\mu<\bar{r}_{\mathrm{low}}/2$, 
robust $(\nu,\mu)$-SVM with $(\nu,\mu)\in\overline{\Lambda}_{\mathrm{low}}$ provides a
classifier with the breakdown point $\mu$. 
The parameter $(\nu,\mu)$ on the outside of $\overline{\Lambda}_{\mathrm{up}}$ is not necessary for the same reasons as
for $\Lambda_{\mathrm{up}}$. 
The acceptable region of $(\nu,\mu)$ is useful when the parameters are determined by a grid search based on cross
validation. The numerical experiments presented in Section~\ref{sec:Numerical_Experiments} applied a grid search to the
region $\overline{\Lambda}_{\mathrm{up}}$.

\section{Asymptotic Properties}
\label{sec:Asymptotic_Properties}
Let us consider the asymptotic properties of robust $(\nu,\mu)$-SVM. 
In the literature~\cite{NECO:Scholkopf+etal:2000}, a uniform bound of the generalization ability of the standard
$\nu$-SVM was calculated for the case that the class of classifiers is properly constrained such that the bias term in
the decision function is bounded in advance. 
Moreover, in~\cite{steinwart03:_optim_param_choic_suppor_vector_machin}, the asymptotic properties of $\nu$-SVM with an
unconstrained parameter space were investigated for a fixed $\nu$. 
To our knowledge, however, the statistical consistency of $\nu$-SVM has not yet been proved. 
The main difficulty comes from the fact that the loss function $-\nu\rho+[\rho-yg(x)]_+$ 
in \eqref{eqn:nu-svm} is not bounded from below. 
Here, therefore, we will study the statistical asymptotic properties of robust $(\nu,\mu)$-SVM on the basis of the
classical asymptotic theory of
L-estimators~\cite{serfling80:_approx_theor_mathem_statis,stigler73:_asymp_distr_trimm_mean}. 

Given a training dataset, the loss function of the robust $(\nu,\mu)$-SVM for a
fixed decision function $g(x)=f(x)+b$ is given by \eqref{eqn:diff_CVaR}. 
The sort of the negative margins,
\begin{align*}
r_{(1)}\geq\cdots\geq{}r_{\sigma(m)}
\end{align*}
for $r_i=-y_ig(x_i),i\in[m]$ 
is called the order statistics, and the linear sum of the order statistics is called the L-estimator. 
The asymptotic properties of L-estimators have been investigated in the field of mathematical
statistics (see \cite[Chap.~22]{vaart00:_asymp_statis}, 
\cite[Chap.~8]{serfling80:_approx_theor_mathem_statis} and references therein for details). 

We will derive the asymptotic distribution of \eqref{eqn:diff_CVaR} with reference to 
\cite{stigler73:_asymp_distr_trimm_mean}. 
Let us define $F_{g}(r)$ as the distribution function of the random variable $R_{g}=-Yg(X)$, 
in which $(X,Y)$ is generated from the population distribution of the training samples. 
Furthermore, the distribution function $G_{g}(r)$ is defined as the conditional probability, 
\begin{align*}
 G_{g}(r)=\Pr\{R_{g}\leq{r}\,|\,\bar{q}_{1-\nu}\leq{R}_{g}<\underline{q}_{1-\mu}\}, 
\end{align*}
where $\bar{q}_{1-\nu}$ and $\underline{q}_{1-\mu}$ are quantiles defined as 
\begin{align*}
 \bar{q}_{1-\nu}&=\sup\{r\,:\,F_{g}(r)\leq{1-\nu}\},\\
 \underline{q}_{1-\mu}&=\inf\{r\,:\,F_{g}(r)\geq{1-\mu}\}, 
\end{align*}
for $0<\mu<\nu<1$. The mean value under the distribution $G_{g}$ is denoted as 
$e_{g}$, i.e., 
\begin{align*}
 e_{g}=\Ebb[R_{g}\,|\,\bar{q}_{1-\nu}\leq{R_{g}}<\underline{q}_{1-\mu}], 
\end{align*}
which is nothing but a trimmed mean of $R_{g}$. 
In addition, let $T_m$ be 
\begin{align*}
 T_m =\frac{1}{(\nu-\mu)m}\sum_{i=\mu{m}+1}^{\nu{m}}\!\!\!\!{}r_{(i)}. 
\end{align*}
According to \cite{stigler73:_asymp_distr_trimm_mean}, the asymptotic distribution of
$\sqrt{m}(T_m-e_{g})$ is expressed by a transformation of a three-dimensional normal
distribution. 
Hence, the random variable $T_m$ converges in probability to $e_{g}$. 
We omit the detailed definition of the asymptotic distribution of $T_m$
(see \cite{stigler73:_asymp_distr_trimm_mean}). 

The asymptotic distribution of $\sqrt{m}(T_m-e_{g})$ has an interesting property. 
Suppose that $\sqrt{m}(T_m-e_{g})$ converges in law to a random variable $Z_{g}$, 
that is distributed from the above asymptotic distribution. 
Let $B_\mu$ be the length of the interval $F_{g}^{-1}(1-\mu)$, as shown in
Figure~\ref{fig:Plot-separate_densities}. When the probability density of  
$F_{g}$ is strictly positive, $B_\mu$ equals zero. 
In addition, suppose that the length of the interval $F_{g}^{-1}(1-\nu)$ is zero. 
Then, the mean value of $Z_{g}$ can be expressed as
\begin{align*}
 \Ebb[Z_{g}]=\frac{B_\mu\sqrt{\mu(1-\mu)}}{\sqrt{2\pi}(\nu-\mu)}, 
\end{align*}
as shown in \cite{stigler73:_asymp_distr_trimm_mean}. 
As a result, the trimmed mean of the negative margins in \eqref{eqn:diff_CVaR} is asymptotically represented as 
\begin{align*}
\phantom{=} (\nu-\mu)\cdot\frac{1}{(\nu-\mu)m}\!\sum_{i=\mu{m}+1}^{\nu{m}}\!\! r_{\sigma(i)} 
=
 (\nu-\mu)e_g
 +O\left(\frac{B_\mu}{\sqrt{m}}\right)
 +O_p\left(\frac{1}{\sqrt{m}}\right), 
\end{align*}
where $O_p(\cdot)$ is the probabilistic order defined in \cite[Chap.~2]{vaart00:_asymp_statis}. 
The above equation is a pointwise approximation at each $(f,b)\in\mathcal{H}\times\Rbb$. 

\begin{figure}[t]
 \begin{center}
  \includegraphics[scale=0.6]{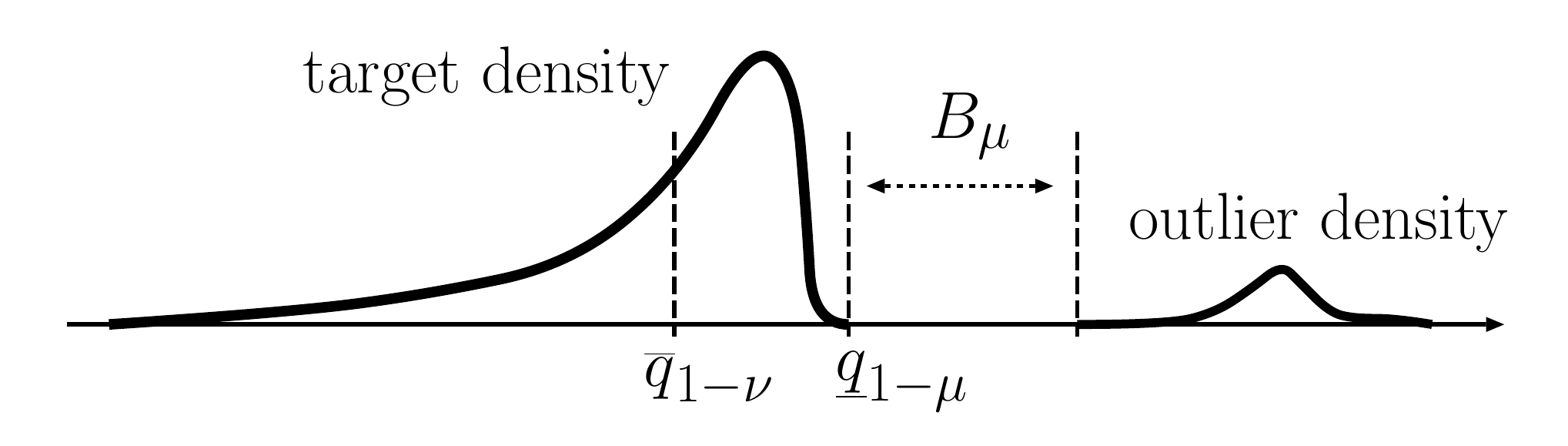}
  \caption{Probability density consisting of two components: the target density and
  outlier density. The gap between two components is the $1-\mu$ quantile of the length $B_\mu$.}
 \label{fig:Plot-separate_densities}
 \end{center}
\end{figure}

In light of the above argument, let us consider the statistical properties of robust $(\nu,\mu)$-SVM. 
Robust $(\nu,\mu)$-SVM in~\eqref{eqn:robust_nu_svm} and ROD in~\eqref{eqn:ROD}
can be made equivalent by appropriately setting the learning parameters $\nu,\mu$ and $\lambda$, 
where the outlier indicator $\eta$ in ROD can take any real number in $[0,1]^m$. 
The discussion in~\cite{xu06:_robus_suppor_vector_machin_train} on setting the parameter $\mu$ in ROD is based on the 
following observation. 
When $\mu$ is small, most $\eta_i$'s take one, and the rest take zero, while, for large $\mu$, all values of $\eta$ fall
below one. This phenomenon is called the second order phase transition in the maximal value of $\eta$. The authors
reported that the phase transition occurs at the value of $\mu$ that corresponds to the true ratio of outliers. 

The above observation is plausible. 
Let us consider robust $(\nu,\mu)$-SVM with $\eta\in[0,1]^m$ that corresponds to ROD with the learning parameters
$\lambda$ and $\mu$. 
Suppose that the probability density of the negative margin, $F_{g}(r)$, is separated into two components as shown in
Fig.~\ref{fig:Plot-separate_densities} and that the true outlier ratio is $\mu_0$. 
When $\mu$ is less than $\mu_0$, there are still some outliers that have not been removed from the training data. 
Hence, negative margins $r_i,i\in[m]$ can take a wide range of real values, and a tie will not occur. 
As a result, the outlier indicator $\eta_i$ tends to take only zero or 
one, even  when $\eta_i$ can take a real number in the interval $[0,1]$. 
If $\mu$ is larger than $\mu_0$, all outliers can be removed from the training data. 
In such a case, 
$\overline{q}_{1-\nu}$ and $\underline{q}_{1-\mu}$ will be close to each other, and some negative margins $r_i$ with
$\eta_i>0$ will concentrate around $\underline{q}_{1-\mu}$. 
If some negative margins take exactly the same value, the outlier indicators on those
samples can take real values in the open interval $(0,1)$. 
Since ROD solves a semidefinite relaxation of the non-convex problem, 
it is conceivable that the numerical solution $\eta$ in ROD will have a similar feature even if the negative margins do
not take exactly the same value.

\section{Numerical Experiments}
\label{sec:Numerical_Experiments}
We conducted numerical experiments on synthetic and benchmark datasets to compare some variants of SVMs. 
The DC algorithm was used to obtain a classifier in the case of robust $(\nu,\mu)$-SVM and robust $C$-SVM using the ramp loss. 
The DC algorithm for robust $C$-SVM is presented in~\cite{collobert06:_tradin}. We used CPLEX to solve the convex quadratic problems. 

\subsection{Breakdown Point}
\label{subsec:Breakdown_Point}
Let us consider the validity of inequality~\eqref{eqn:key_inequality} in
Theorem~\ref{theorem:breakdown-point-optvalue}. 
In the numerical experiments, the original data $D$ was generated using {\tt mlbench.spirals}
in the mlbench library of the R language~\cite{team14:_r}. 
Given an outlier ratio $\mu$, positive samples of size $\mu{m}$ were randomly chosen from
$D$, and they were replaced with randomly generated outliers 
to obtain a contaminated dataset $D'\in\mathcal{D}_{\mu{m}}$. 
The original data $D$ and an example of the contaminated data $D'\in\mathcal{D}_{\mu{m}}$
are shown in Fig.~\ref{fig:org_contam_data}. 
The decision function $g(x)=f(x)+b$ was estimated from $D'$ by using robust $(\nu,\mu)$-SVM. 
Here, the true outlier ratio $\mu$ was used as the parameter of the learning algorithm. 
The norms of $f$ and $b$ were then evaluated. 
The above process was repeated 30 times for each parameters $(\nu,\mu)$, and the maximum value of $\|f\|_{\mathcal{H}}$ and $|b|$
was computed. 

\begin{figure}[t]
\begin{center}
\begin{tabular}{cc}
 original data $D$  &
 \quad{} example of $D'\in\mathcal{D}_{\mu{m}}$\\  
 \includegraphics[scale=0.32]{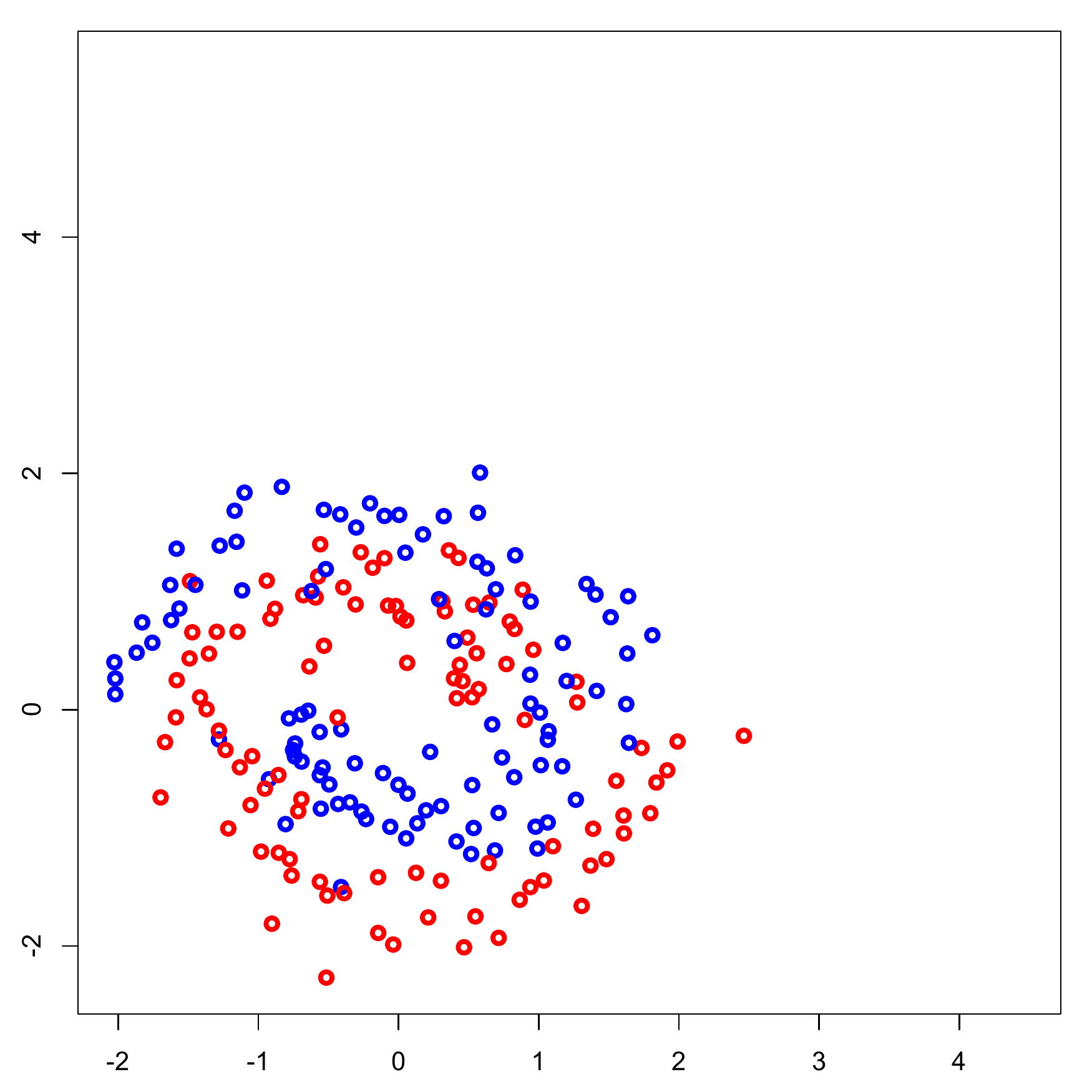}\quad & \quad
 \includegraphics[scale=0.32]{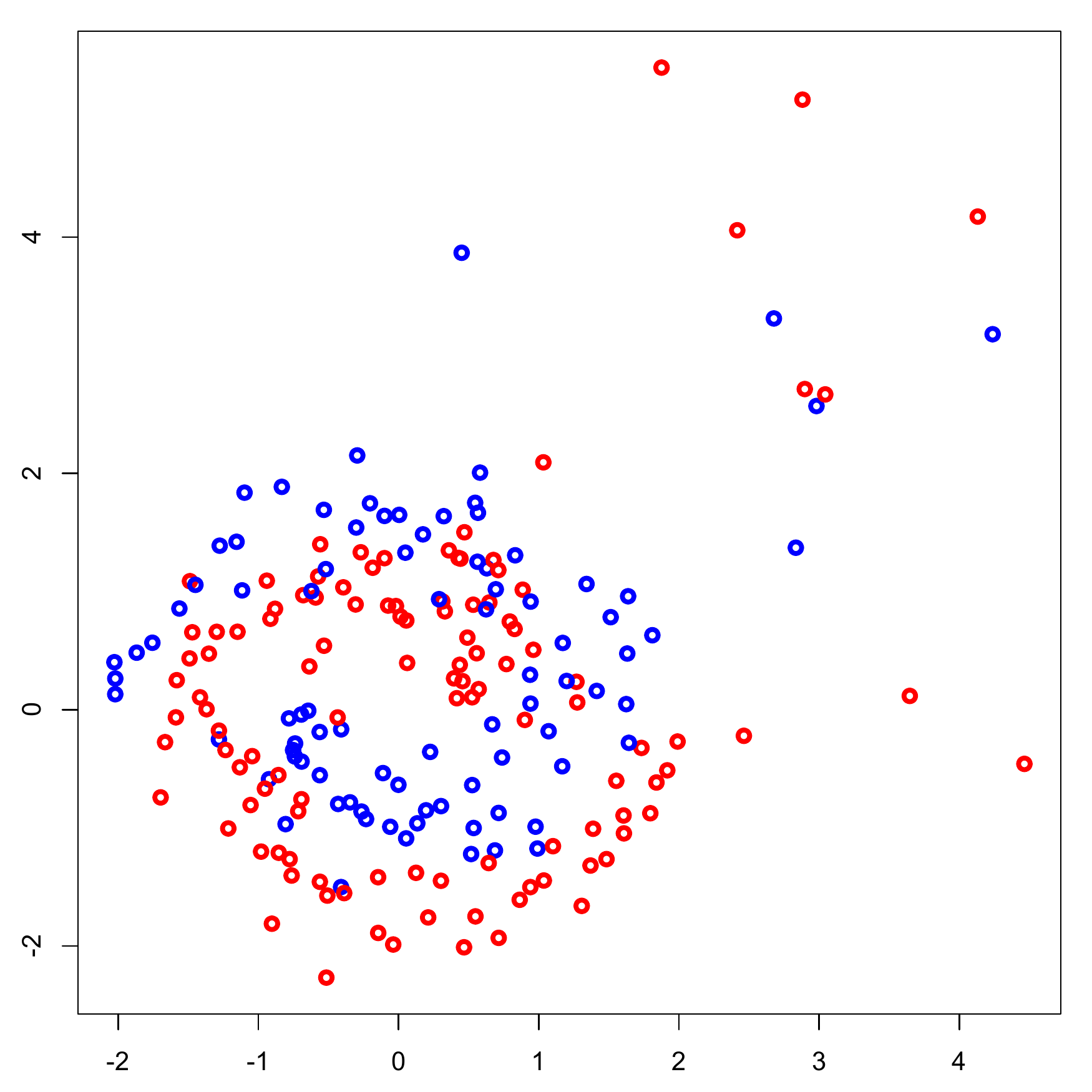}
\end{tabular}
 \caption{The left panel shows the original data $D$, and the right panel shows the contaminated data
 $D'\in\mathcal{D}_{\mu{m}}$. In this example, the sample size is $m=200$, and  
 the outlier ratio is $\mu=0.1$. }
 \label{fig:org_contam_data}
\end{center}
\end{figure}

Figure~\ref{fig:estimator_and_testerror} shows the results of the numerical experiments. 
The maximum norm of the estimated decision function is plotted for the parameter
$(\mu,\nu-\mu)$ on the same axis as Fig.~\ref{fig:breakdown-ponit}. 
The top (bottom) panels show the results for a Gaussian (linear) kernel. 
The left and middle columns show the maximum norm of $f$ and $b$, respectively. 
The maximum test errors are presented in the right column. 
In all panels, the red points denote the top 50 percent of values, 
and the asterisk ($\ast$) is the point that violates the inequality 
$\nu-\mu\leq2(r-2\mu)$. 
In this example, the numerical results agree with the theoretical analysis in Section~\ref{sec:Breakdown_Point}; i.e., the norm
becomes large when the inequality $\nu-\mu\leq 2(r-2\mu)$ is violated. 
Accordingly, the test error gets close to $0.5$---no information for classification. Even when the unbounded linear kernel is
used, robustness is confirmed for the parameters in the left lower region in the right panel of Fig.~\ref{fig:breakdown-ponit}. 

In the bottom right panel, the test error gets large when the inequality
$\nu-\mu\leq2(r-2\mu)$ holds. This result comes from the problem setup. 
Even with non-contaminated data, the test error of the standard $\nu$-SVM is approximately $0.5$,
because the linear kernel works poorly for spiral data. Thus, the worst-case test error can go beyond $0.5$. 
For the parameter at which \eqref{eqn:key_inequality} is violated, the test error is always close to $0.5$. 
Thus, a learning method with such parameters does not provide any useful information for classification. 

\begin{figure}
\begin{center}
\hspace*{-10mm}
\begin{tabular}{ccc}
 \multicolumn{3}{c}{(a) Gaussian kernel} \\
 \includegraphics[scale=0.32]{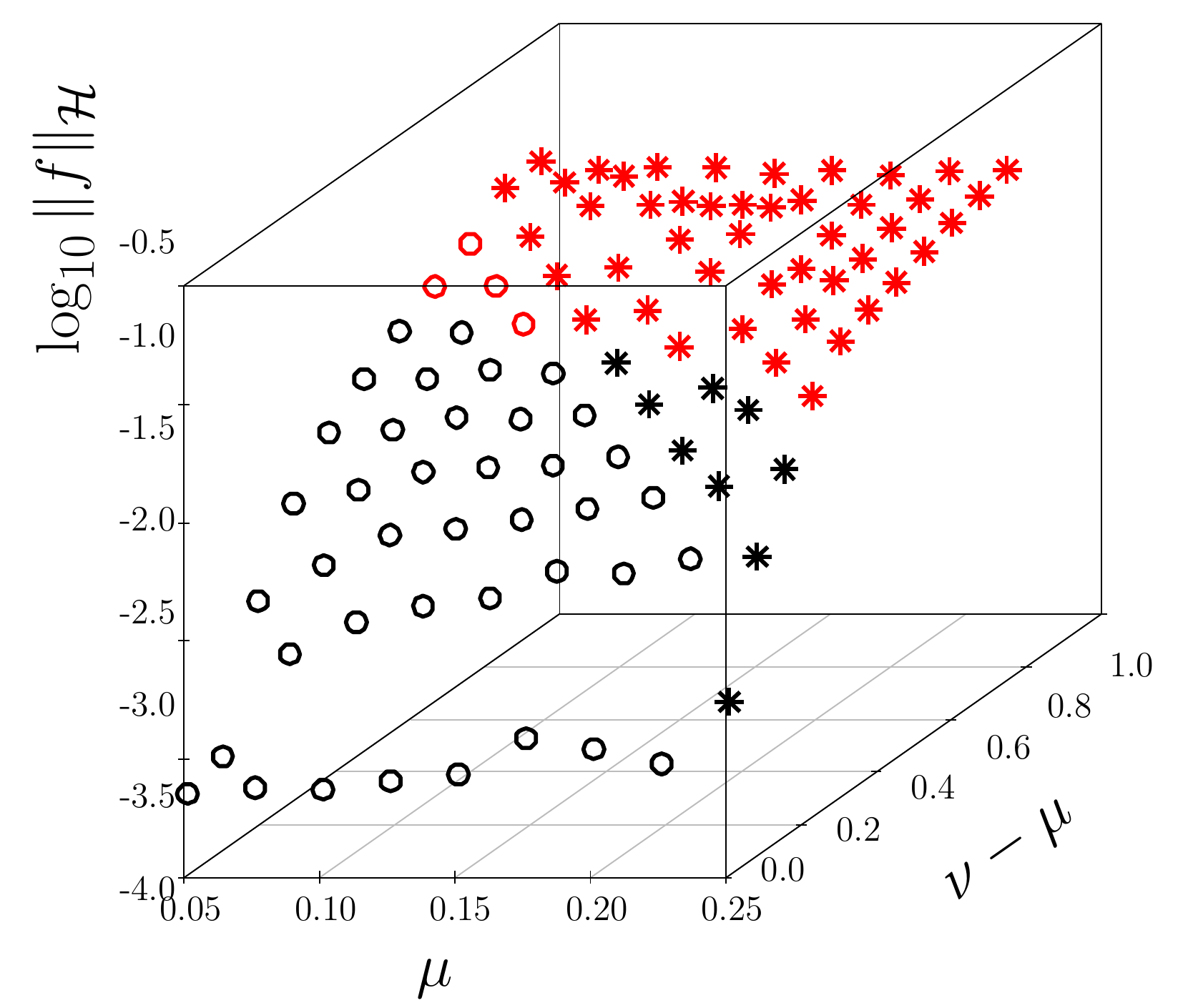} & 
 \includegraphics[scale=0.32]{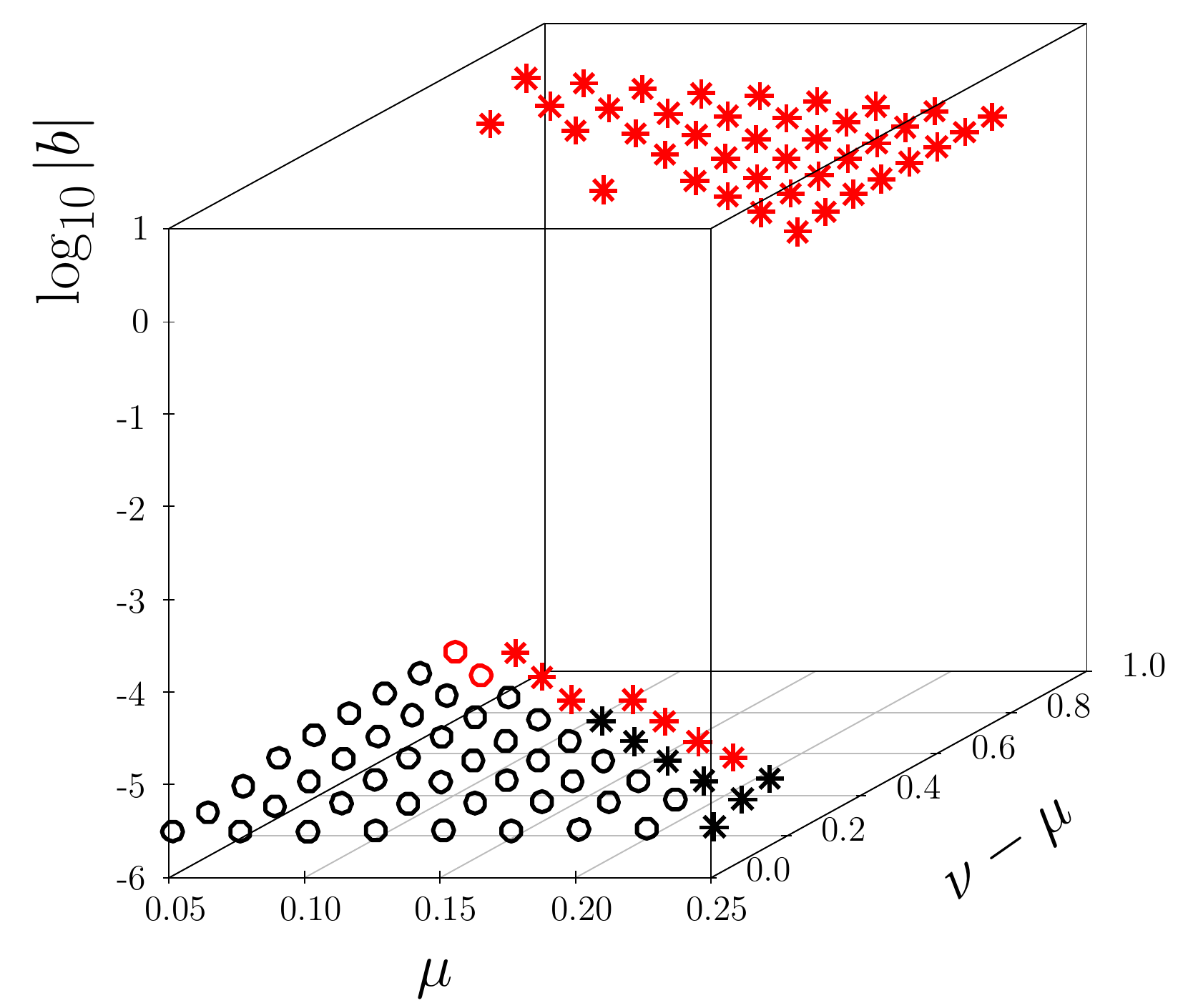} & 
 \includegraphics[scale=0.32]{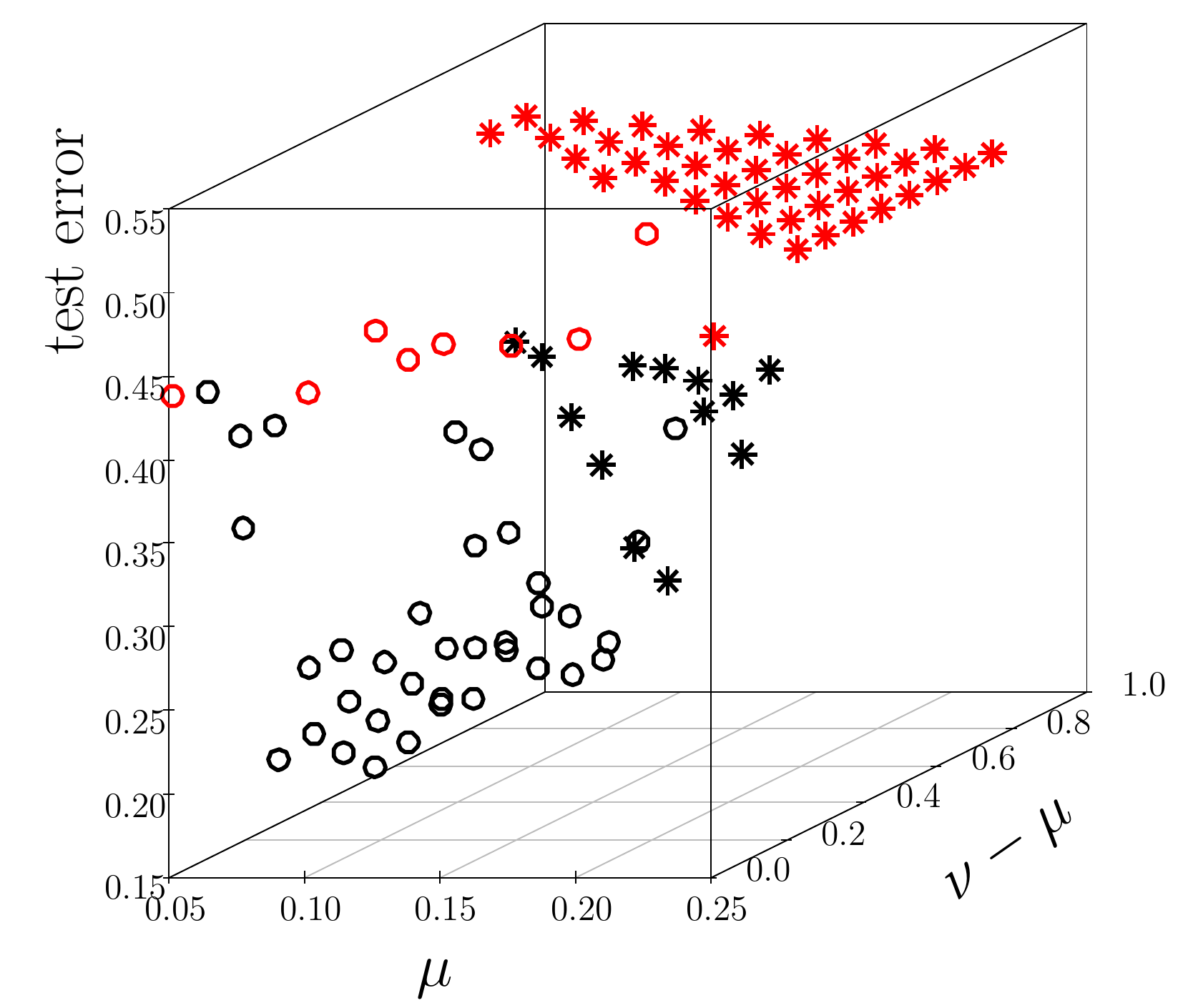} \\
 plot of $\log_{10}\|f\|_{\mathcal{H}}$  &  plot of $\log_{10}|b|$  &   plot of  test error\\ \\
 \multicolumn{3}{c}{(b) Linear kernel} \\
 \includegraphics[scale=0.32]{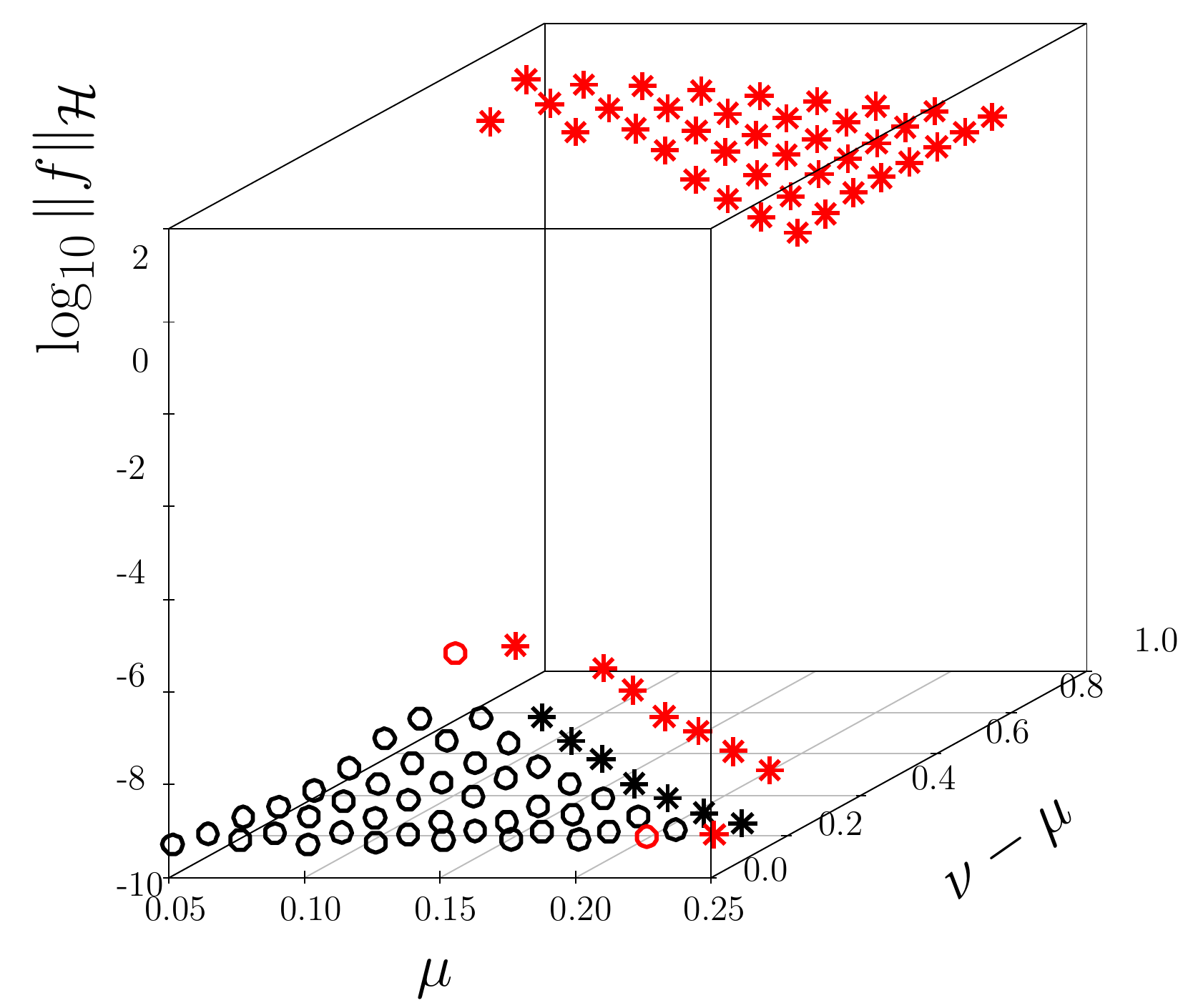} & 
 \includegraphics[scale=0.32]{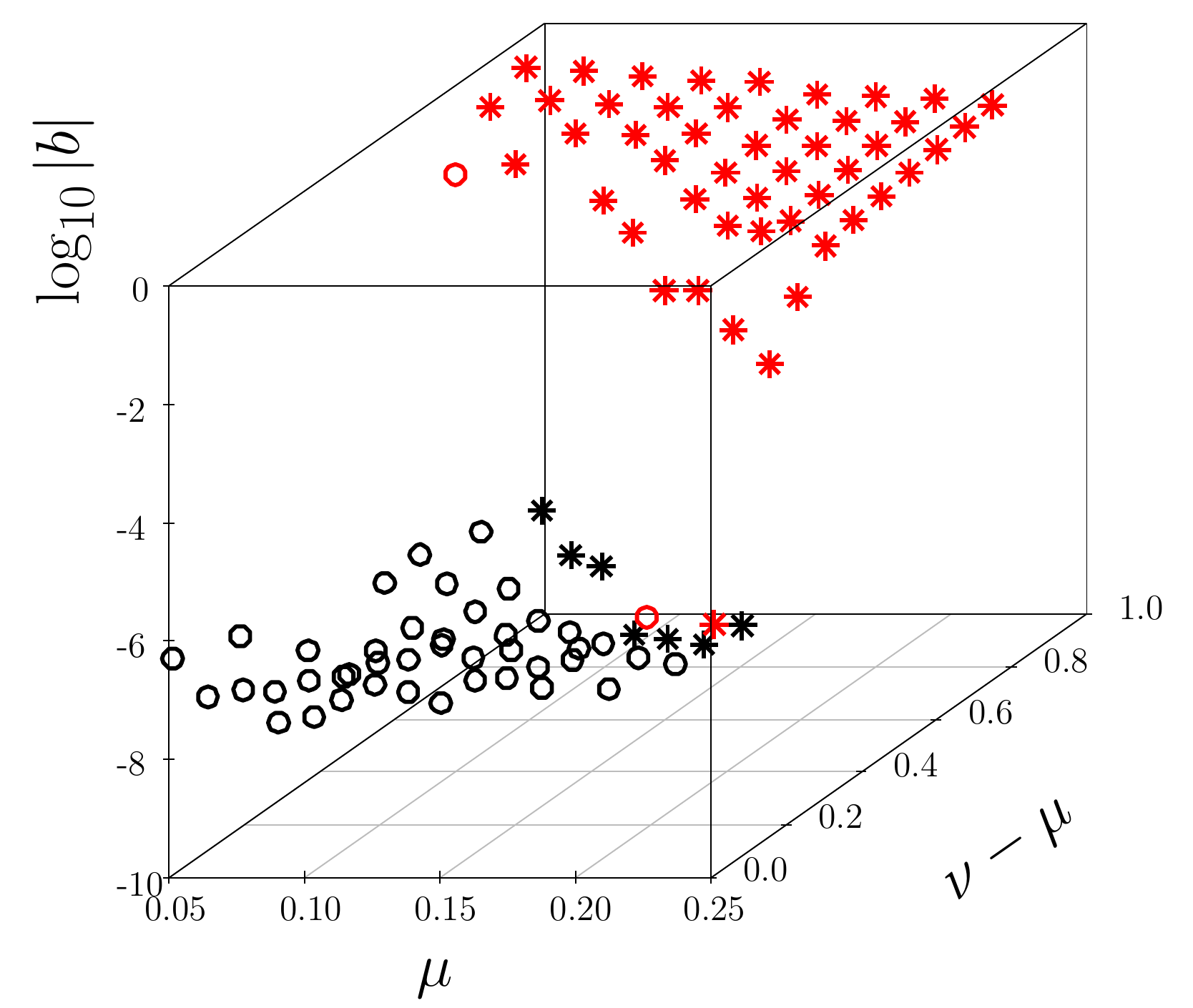} & 
 \includegraphics[scale=0.32]{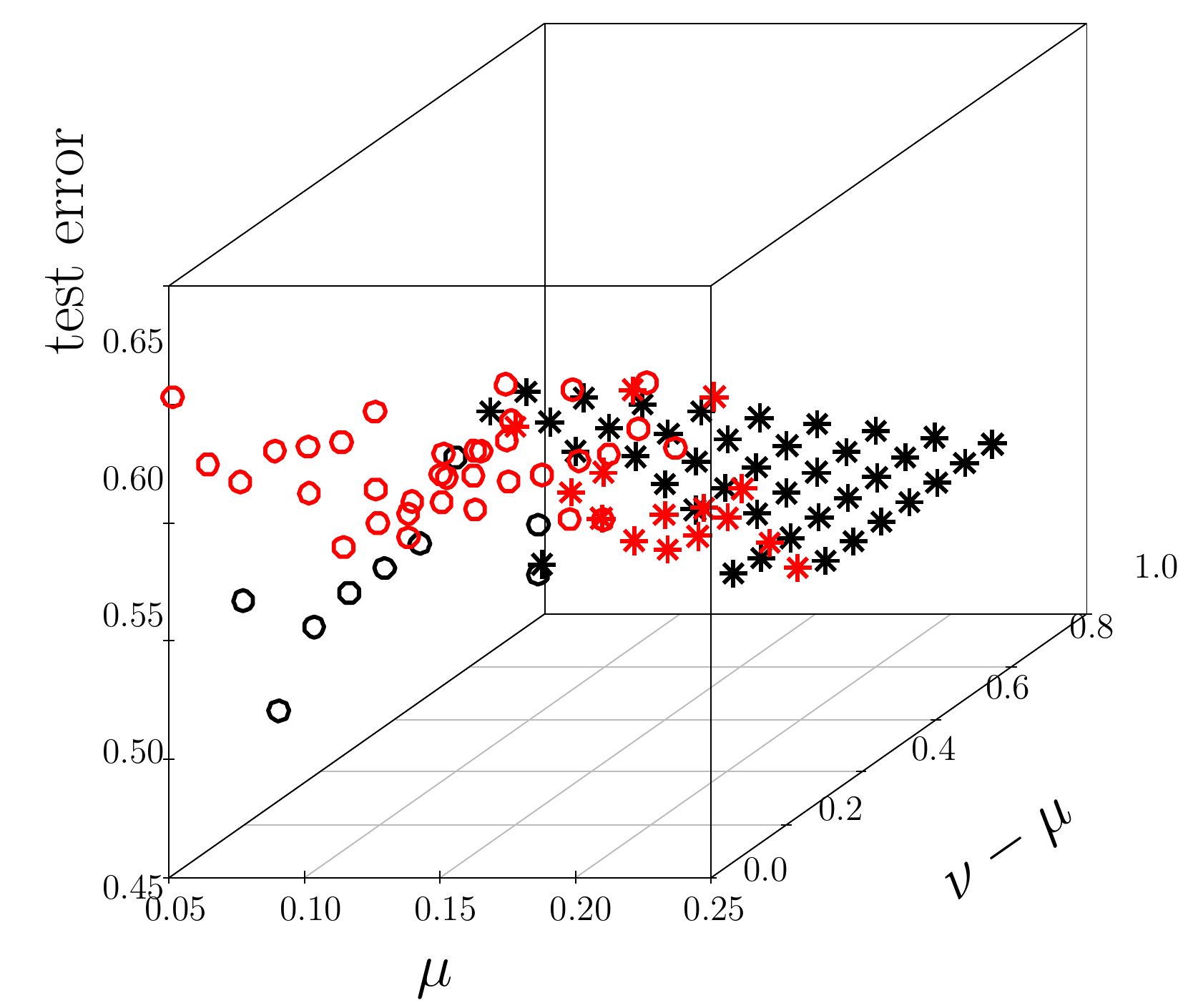} \\
 plot of $\log_{10}\|f\|_{\mathcal{H}}$  &  plot of $\log_{10}|b|$  &  plot of test error\\
\end{tabular}
\end{center}
 \caption{Plots of the maximum norms and the worst-case test errors. 
 The top (Bottom) panels show the results for a Gaussian (linear) kernel. 
 Red points mean the top 50 percent of values, and the asterisk ($\ast$) is the point
 that violates the inequality $\nu-\mu\leq 2(r-2\mu)$. } 
 \label{fig:estimator_and_testerror}
\end{figure}

\subsection{Prediction Accuracy}
\label{subsec:Prediction_Accuracy}
We compared the generalization ability of the robust $(\nu,\mu)$-SVM with existing classifiers such as standard
$\nu$-SVM and robust $C$-SVM using the ramp loss. The datasets are presented in Table~\ref{table:sim_result}. 
All the datasets are provided in the {\tt mlbench} and {\tt kernlab} libraries of the R language~\cite{team14:_r}. 
In all the datasets, the number of positive samples is less than or equal to that of negative samples. 
Before running the learning algorithms, we standardized each input variable with mean zero and standard deviation one. 

We randomly split the dataset into training and test sets. To evaluate the robustness, the training data was
contaminated by outliers. More precisely, we randomly chose positive labeled samples in the training data and changed
their labels to negative; i.e., we added outliers by flipping the labels. After that, robust $(\nu,\mu)$-SVM, robust
$C$-SVM using the ramp loss, and the standard $\nu$-SVM were used to obtain classifiers from the contaminated training
dataset. The prediction accuracy of each classifier was then evaluated over test data that had no outliers. Linear and
Gaussian kernels were employed for each learning algorithm. 

The learning parameters, such as $\mu,\nu$, and $C$, were determined by conducting a grid search based on five-fold cross
validation over the training data.  
For robust $(\nu,\mu)$-SVM, the parameter $(\mu,\nu)$ was selected from the region $\overline{\Lambda}_{\mathrm{up}}$ in
\eqref{eqn:worst-case_par_region}.  
For standard $\nu$-SVM, the candidate of the regularization parameter $\nu$ was selected
from the interval $(0,2r')$, where $r'$ is the label ratio of the contaminated training data. 
For robust $C$-SVM, the regularization parameter $C$ was selected from the interval $[10^{-7},\,10^{7}]$. 
In the grid search of the parameters, 24 or 25 candidates were examined for each learning method. 
Thus, we needed to solve convex or non-convex optimization problems more
than $24\times5$ times in order to obtain a classifier. The above process was repeated 30 times, 
and the average test error was calculated. 

The results are presented in Table~\ref{table:sim_result}. 
For non-contaminated training data, robust $(\nu,\mu)$-SVM and robust $C$-SVM were comparable to the standard $\nu$-SVM. 
When the outlier ratio is high, we can conclude that robust $(\nu,\mu)$-SVM and robust $C$-SVM tend to work better than
the standard $\nu$-SVM. 
In this experiment, the kernel function does not affect the relative prediction performance of these learning methods. 
In large datasets such as spam and Satellite, robust $(\nu,\mu)$-SVM tends to outperform robust $C$-SVM. 
When learning parameters, such as $\nu,\mu$, and $C$, are appropriately chosen by using a large dataset, 
learning algorithms with plural learning parameters clearly work better than those with a single learning parameter. 
In addition, in robust $C$-SVM, there is a difficulty in choosing the regularization parameter. 
Indeed, the parameter $C$ does not have a clear meaning, and thus, it is not straightforward to determine the candidates 
of $C$ in the grid search optimization. 
In contrast, the parameter $\nu$ in $\nu$-SVM and its robust variant
has a clear meaning, i.e., a lower bound of the ratio of support vectors and an upper 
bound of the margin error on the training data~\cite{NECO:Scholkopf+etal:2000}. 
Such clear meaning is of great help to choose candidate points of regularization parameters. 

We conducted another experiment in which the learning parameters $\nu,\mu$ and $C$ were determined using only one
validation set, i.e., non-cross validation (the details are not presented here). 
The dataset was split into training, validation and test sets. The learning parameters, $\nu, \mu$, and $C$, that minimized 
the prediction error on the validation set were selected.  
This method greatly reduced the computational cost of the cross validation. 
However, robust $(\nu,\mu)$-SVM did not necessarily produce a better classifier compared with the other methods. 
Since robust $(\nu,\mu)$-SVM has two learning parameters, we need to carefully select them using cross validations rather than
simple validations in order to achieve high prediction accuracy.

\begin{table}[p]
 \caption{
 Test error and standard deviation of robust $(\nu,\mu)$-SVM, robust $C$-SVM and $\nu$-SVM. 
 The dimension of the input vector, the number of training samples, the number of test
 samples, and the label ratio of all samples with no outliers are shown for each dataset. 
 Linear and Gaussian kernels were used to build the classifier in each method. 
 The outlier ratio in the training data ranged from 0\% to 15\%, and the test error was evaluated on
 the non-contaminated test data. The asterisk $(\ast)$ means the best result for a fixed kernel function in each dataset, 
 and the double asterisks $(\ast\ast)$ mean that the corresponding method is 5\%
 significant compared with the second best method under a one-sided t-test. 
 The learning parameters were determined by five-fold cross validation on the contaminated training data.}
 \label{table:sim_result}
 \begin{center}\footnotesize
 \hspace*{-12mm}
 \begin{tabular}{rccccccc}
  \multicolumn{7}{l}{Sonar: $\dim{x}=60$, $\#$train=104, $\#$test=104, $r=0.466$.}\\\hline
 &\multicolumn{3}{c}{Linear kernel} && \multicolumn{3}{c}{Gaussian kernel} \\
 outlier
 & \shortstack{robust\\ $(\nu,\mu)$-SVM} & \shortstack{robust\\ $C$-SVM}  & $\nu$-SVM &
 & \shortstack{robust\\ $(\nu,\mu)$-SVM} & \shortstack{robust\\ $C$-SVM}  & $\nu$-SVM \\
 \cline{2-4}  \cline{6-8}
 0\%&\phantom{*}.258(.032)&.270(.038)&          *.256(.051)&&          *.179(.038) &\phantom{**}.188(.043) &\phantom{*}.181(.039)\\
 5\%&          *.256(.039)&.273(.047)&\phantom{*}.258(.046)&&\phantom{*}.225(.042) &\phantom{**}.229(.051) &          *.224(.061)\\
10\%&          *.297(.060)&.306(.067)&\phantom{*}.314(.060)&&\phantom{*}.249(.059) &          **.230(.046) &\phantom{*}.259(.062)\\
15\%&          *.329(.061)&.339(.064)&\phantom{*}.345(.062)&&\phantom{*}.280(.053) &\phantom{*}*.280(.050) &\phantom{*}.294(.064)\\ \\
\multicolumn{7}{l}{BreastCancer: $\dim{x}=10$, $\#$train=350, $\#$test=349, $r=0.345$. }\\ \hline
 &\multicolumn{3}{c}{Linear kernel} && \multicolumn{3}{c}{Gaussian kernel} \\ 
 outlier
 & \shortstack{robust\\ $(\nu,\mu)$-SVM} & \shortstack{robust\\ $C$-SVM}  & $\nu$-SVM &
 & \shortstack{robust\\ $(\nu,\mu)$-SVM} & \shortstack{robust\\ $C$-SVM}  & $\nu$-SVM \\
 \cline{2-4}  \cline{6-8}
 0\%&.033(.010) &\phantom{*}.035(.008) &*.033(.006)&&\phantom{*}*.032(.008)&\phantom{*}.035(.012) &.033(.010)\\
 5\%&.034(.009) &*.034(.010) &\phantom{*}.043(.015)&&\phantom{*}*.032(.005)&\phantom{*}.033(.007) &.033(.006)\\
10\%&.055(.015) &*.051(.026) &\phantom{*}.076(.036)&&          **.035(.008)&\phantom{*}.043(.025) &.038(.008)\\
15\%&.136(.058) &*.120(.050) &\phantom{*}.148(.058)&&\phantom{**}.160(.083)&          *.145(.070) &.150(.110)\\ \\
  \multicolumn{7}{l}{PimaIndiansDiabetes: $\dim{x}=8$, $\#$train=384, $\#$test=384, $r=0.349$. }\\ \hline
 &\multicolumn{3}{c}{Linear kernel} && \multicolumn{3}{c}{Gaussian kernel} \\
 outlier
 & \shortstack{robust\\ $(\nu,\mu)$-SVM} & \shortstack{robust\\ $C$-SVM}  & $\nu$-SVM &
 & \shortstack{robust\\ $(\nu,\mu)$-SVM} & \shortstack{robust\\ $C$-SVM}  & $\nu$-SVM \\
 \cline{2-4}  \cline{6-8}
 0\%&\phantom{**}.237(.018) &*.232(.014) &.246(.018)&&*.238(.021) &\phantom{*}.240(.019) &.243(.022)\\
 5\%&\phantom{**}.239(.019) &*.237(.016) &.269(.036)&&*.264(.025) &\phantom{*}.267(.024) &.273(.024)\\
10\%&**.280(.046) &\phantom{*}.299(.042) &.330(.030)&&\phantom{*}.302(.039) &*.293(.036) &.315(.038)\\
15\%&**.338(.042) &\phantom{*}.349(.030) &.351(.026)&&*.344(.028) &\phantom{*}.344(.031) &.353(.016)\\ \\
  \multicolumn{7}{l}{spam: $\dim{x}=57$, $\#$train=1000, $\#$test=3601, $r=0.394$. }\\ \hline
 &\multicolumn{3}{c}{Linear kernel} && \multicolumn{3}{c}{Gaussian kernel} \\
 outlier
 & \shortstack{robust\\ $(\nu,\mu)$-SVM} & \shortstack{robust\\ $C$-SVM}  & $\nu$-SVM &
 & \shortstack{robust\\ $(\nu,\mu)$-SVM} & \shortstack{robust\\ $C$-SVM}  & $\nu$-SVM \\ 
 \cline{2-4}  \cline{6-8}
 0\%&\phantom{**}.083(.005) &.088(.006) &*.083(.005)           &&\phantom{**}.081(.005) &.086(.006) &*.081(.006)\\
 5\%&          **.094(.008) &.104(.013) &\phantom{*}.109(.010) &&\phantom{**}.095(.008) &.097(.009) &*.095(.008)\\
10\%&          **.129(.022) &.152(.020) &\phantom{*}.166(.067) &&\phantom{*}*.129(.015) &.133(.017) & .141(.030)\\
15\%&          **.201(.029) &.240(.030) &\phantom{*}.256(.091) &&          **.206(.018) &.223(.030) & .240(.055)\\ \\
  \multicolumn{7}{l}{Satellite: $\dim{x}=36$, $\#$train=2000, $\#$test=4435, $r=0.234$. }\\ \hline
 &\multicolumn{3}{c}{Linear kernel} && \multicolumn{3}{c}{Gaussian kernel} \\
 outlier
 & \shortstack{robust\\ $(\nu,\mu)$-SVM} & \shortstack{robust\\ $C$-SVM}  & $\nu$-SVM &
 & \shortstack{robust\\ $(\nu,\mu)$-SVM} & \shortstack{robust\\ $C$-SVM}  & $\nu$-SVM \\ 
 \cline{2-4}  \cline{6-8}
 0\%&\phantom{**}.097(.004) &\phantom{*}.096(.003) &          **.094(.003) && \phantom{*}.069(.031) &.067(.004)           &**.063(.004)\\
 5\%&\phantom{**}.101(.003) &          *.100(.005) &\phantom{**}.100(.004) &&           *.072(.015) &.078(.007) &\phantom{**}.078(.043)\\
10\%&          **.148(.020) &\phantom{*}.161(.026) &\phantom{**}.161(.019) &&           *.117(.034) &.126(.040) &\phantom{**}.137(.027)\\
 \end{tabular}
 \end{center}
\end{table}

\section{Concluding Remarks}
\label{sec:Concluding_Remarks}
We presented robust $(\nu,\mu)$-SVM and studied its statistical properties. 
The robustness property was analyzed by computing the exact breakdown point. 
As a result, we obtained inequalities for the learning parameters $\nu$ and $\mu$
that guarantee the robustness of the learning algorithm. 
The statistical theory of the L-estimator was then used to investigate the asymptotic behavior of the classifier. 
Numerical experiments showed that the inequalities are critical to obtaining a robust classifier.  
The prediction accuracy of the proposed method was numerically compared with those of other methods, 
and it was found that the proposed method with carefully chosen learning parameters delivers
more robust classifiers than those of other methods such as standard $\nu$-SVM and robust $C$-SVM using the ramp loss. 
In the future, we will explore the robustness properties of more general learning methods. 
Another important issue is to develop efficient optimization algorithms. 
Although the DC algorithm~\cite{collobert06:_tradin,Pham:1997} and convex
relaxation~\cite{xu06:_robus_suppor_vector_machin_train,yu12:_polyn_form_robus_regres} are
promising methods, more scalable algorithms will be required to deal with massive datasets 
that are often contaminated by outliers. 

\appendix

\section{Derivation of DC Algorithm}
\label{appendix:DCA}
According to \eqref{eqn:DC-representation}, the objective function of the robust $(\nu,\mu)$-SVM is expressed as
$\Phi({\alpha},b,\rho)=\psi_0({\alpha},b,\rho)-\psi_1({\alpha},b)$ using the convex functions $\psi_0$ and $\psi_1$
defined as  
\begin{align*}
 \psi_0({\alpha},b,\rho)&=\frac{1}{2}{\alpha}^TK{\alpha}-\nu\rho+\frac{1}{m}\sum_{i=1}^{m}[\rho+r_i],\\
 \psi_1({\alpha},b)&=\max_{{\eta}\in{E_\mu}}\frac{1}{m}\sum_{i=1}^{m}(1-\eta_i)r_i, 
\end{align*}
where $r_i$ is the negative margin $r_i=-y_i(\sum_{i=1}^{m}K_{ij}\alpha_i+b)$ and $K\in\Rbb^{m\times{m}}$ is the Gram
matrix defined by $K_{ij}=k(x_i,x_j),\,i,j\in[m]$. Let ${\alpha}_t,b_t,\rho_t$ be the solution obtained after $t$ 
iterations of the DC algorithm. Then, the solution is updated to the optimal solution of 
\begin{align}
 \label{eqn:DC_primal}
\min_{{\alpha},b,\rho} \psi_0({\alpha},b,\rho)-u^T{\alpha}-vb,
\end{align}
where $(u,v)\in\Rbb^{m+1}$ with $u\in\Rbb^m, v\in\Rbb$ is an element of the subgradient of $\psi_1$ at $({\alpha}_t,b_t)$. 
The subgradient of $\psi_1$ is given as 
\begin{align*}
 &\phantom{=}\partial\psi_1({\alpha}_t,b_t)\\
 &=
 \mathrm{conv}
 \bigg\{
 (u,v)\,:\,
 u=-\frac{1}{m}K({{y}}\circ(1_m-{\eta})),\,
 v=-\frac{1}{m}{{y}}^T(1_m-{\eta}),\\
 &\phantom{= \mathrm{conv}\bigg\{}\ \text{where 
 ${\eta}$ is a maximum solution of the problem in $\psi_1({\alpha}_t,b_t)$}
 \bigg\}, 
\end{align*}
where $\mathrm{conv}S$ denotes the convex hull of the set $S$. 
As shown in Algorithm~\ref{alg:DC_alg_robust_numuSVM}, a parameter ${\eta}$ that meets the condition in the above
subgradient is obtained from the sort of the negative margin of the decision function defined from
$({\alpha}_t,b_t)$. 
The dual problem of \eqref{eqn:DC_primal} is presented in \eqref{eqn:DC_alg_quadprob}, 
up to a constant term that is independent of the optimization parameter.

\section{Proofs of Theorems} 

\subsection{Proof of Theorem~\ref{theorem:breakdown-point-optvalue}}
\label{appendix:proof_breakdown_obj}

The proof is decomposed into two lemmas. Lemma~\ref{lemma:breakdown-point-ER-SVM} shows that
condition (i) is sufficient for condition (ii), 
and Lemma~\ref{lemma:unbounded-breakdown-point-ER-SVM-1} shows that condition (ii) does not hold if inequality
\eqref{eqn:key_inequality} is violated. 
For the dataset $D=\{(x_i,y_i):i\in[m]\}$, let $I_{+}$ and $I_{-}$ be the index sets defined as 
$I_{\pm}=\{i:y_i=\pm1\}$. When the parameter $\mu$ is equal to zero, the theorem holds
according to the argument on the standard $\nu$-SVM~\cite{CriBur00}.  
Below, we assume $\mu>0$. 

\begin{lemma}
 \label{lemma:breakdown-point-ER-SVM}
 Under the assumptions of Theorem~\ref{theorem:breakdown-point-optvalue}, 
 condition (i) leads to condition (ii). 
\end{lemma}

\begin{proof}
 [Proof of Lemma \ref{lemma:breakdown-point-ER-SVM}]
 We show that $\mathcal{V}_\eta[\nu,\mu;D']$ is not empty for any
 $D'\in\mathcal{D}_{\mu{m}}$. 
 For a parameter $\mu$ such that $\mu<r/2$, let $c>0$ be a positive constant
 satisfying $\mu=r/(c+2)$. 
 Then, \eqref{eqn:key_inequality} 
 is expressed as $\nu\leq(r+2cr)/(c+2)$. 
 For a contaminated dataset $D'=\{(x_i',y'_i):i\in[m]\}\in\mathcal{D}_{\mu{m}}$, 
 let us define $\widetilde{I}_{+}\subset{}I_{+}$ as an index set such that 
 $(x_i,y_i)\in{D}$ for $i\in{}\widetilde{I}_{+}$ is replaced with $(x_i',y_i')\in{D'}$ as
 an outlier. In the same way, $\widetilde{I}_{-}\subset{}I_{-}$ is defined for negative
 samples in $D$. 
 Therefore, for any index $i$ in 
 $I_{+}\setminus\widetilde{I}_{+}$ or $I_{-}\setminus\widetilde{I}_{-}$, 
 we have $(x_i,y_i)=(x_i',y_i')$. 
 The assumptions of the theorem ensure
 $|\widetilde{I}_{+}|+|\widetilde{I}_{-}|\leq\mu{m}$. 
 From $\mu{}m={\min\{|I_+|,|I_-|\}/(c+2)}$, we obtain 
 \begin{align*}
  |\widetilde{I}_{+}|&\leq{}|I_{+}|/(c+2)<(c+1)|I_{+}|/(c+2)\leq|I_{+}\setminus\widetilde{I}_{+}|,\\
  |\widetilde{I}_{-}|&\leq{}|I_{-}|/(c+2)<(c+1)|I_{-}|/(c+2)\leq|I_{-}\setminus\widetilde{I}_{-}|. 
 \end{align*}
 Given ${\eta}\in{E}_\mu$, the sets $\widetilde{\mathcal{U}}_{\eta}^{+}[\nu,\mu;D']$ and
 $\widetilde{\mathcal{U}}_{\eta}^{-}[\nu,\mu;D']$ are defined by 
 \begin{align*}
  &\phantom{=}
  \widetilde{\mathcal{U}}_{\eta}^{\pm}[\nu,\mu;D']\\
  &=
  \bigg\{  \sum_{i:y_i'=\pm1}\gamma_i'k(\cdot,x_i')\in\mathcal{H}\,:\,\sum_{i:y_i'=\pm1}\!\gamma_i'=1,\ 
  0\leq \gamma_i'\leq \frac{2\eta_i}{(\nu-\mu)m}, \ \ 
  \text{$\gamma_i'=0$ for $i\not\in{}I_{\pm}\setminus\widetilde{I}_{\pm}$}
  \bigg\}. 
 \end{align*}
 Note that 
 $\widetilde{\mathcal{U}}_{\eta}^{\pm}[\nu,\mu;D']\subset\mathcal{U}_{\eta}^{\pm}[\nu,\mu;D']$
 holds because of the additional constraint, 
 $\gamma_i'=0$ for $i\not\in{}I_{\pm}\setminus\widetilde{I}_{\pm}$. 
 In addition, we have
 $\widetilde{\mathcal{U}}_{\eta}^{\pm}[\nu,\mu;D'] \subset \mathrm{conv}\{k(\cdot,x_i)\,:\,i\in{}I_{\pm}\}$, 
 since only the element $k(\cdot,x_i')$ with $i\in{}I_{\pm}\setminus\widetilde{I}_{\pm}$,
 i.e. $x_i'=x_i$, can have a non-zero coefficient $\gamma_i'$ in
 $\widetilde{\mathcal{U}}_{\eta}^{\pm}[\nu,\mu;D']$. 

 We prove that $\widetilde{\mathcal{U}}_{\eta}^{+}[\nu,\mu;D']$ and $\widetilde{\mathcal{U}}_{\eta}^{-}[\nu,\mu;D']$ are
 not empty. 
 The size of the index sets $\{i\in{}I_{\pm}\setminus\widetilde{I}_{\pm}\,:\,\eta_i=1\}$
 is bounded below by
 \begin{align*}
  |\{i\in{}I_{\pm}\setminus\widetilde{I}_{\pm}\,:\,\eta_i=1\}|  \geq
  \frac{(c+1)|I_{\pm}|}{c+2}-\mu{m}\geq  \frac{c|I_{\pm}|}{c+2} >0. 
 \end{align*}
 The inequality $\nu-\mu\leq2(r-2\mu)$ is equivalent to 
 $1\leq2cr/((\nu-\mu)(c+2))$. Hence, we have 
 \begin{align*}
  1\leq\frac{2cr}{(\nu-\mu)(c+2)}\leq  \frac{2}{(\nu-\mu)m}\cdot\frac{c|I_{\pm}|}{c+2}
  \leq
  \frac{2}{(\nu-\mu)m}\cdot|\{i\in{}I_{\pm}\setminus\widetilde{I}_{\pm}\,:\,\eta_i=1\}|, 
 \end{align*}
 implying that the sets $\widetilde{\mathcal{U}}_{\eta}^{\pm}[\nu,\mu;D']$ are not empty. 
 Indeed, the coefficients defined by 
 \begin{align*}
  \gamma_j'= \frac{1}{|\{i\in{}I_{\pm}\setminus\widetilde{I}_{\pm}\,:\,\eta_i=1\}|}\leq \frac{2}{(\nu-\mu)m}
 \end{align*}
 for $j\in\{i\in{}I_{\pm}\setminus\widetilde{I}_{\pm}\,:\,\eta_i=1\}$ 
 and otherwise $\gamma_j'=0$ admit all the constraints in  $\widetilde{\mathcal{U}}_{\eta}^{\pm}[\nu,\mu;D']$. 
 Since 
 $\emptyset\neq\widetilde{\mathcal{U}}_{\eta}^{\pm}[\nu,\mu;D']\subset\mathcal{U}_{\eta}^{\pm}[\nu,\mu;D']$
 holds, we have 
 $\emptyset\neq\widetilde{\mathcal{V}}_{\eta}[\nu,\mu;D']\subset\mathcal{V}_{\eta}[\nu,\mu;D']$, 
 where
 $\widetilde{\mathcal{V}}_{\eta}[\nu,\mu;D']
 =
 \widetilde{\mathcal{U}}_{\eta}^{+}[\nu,\mu;D']\ominus\widetilde{\mathcal{U}}_{\eta}^{-}[\nu,\mu;D']$. 

 Now, let us prove the inequality
 \begin{align}
  \label{eqn:inequalities_tobe_prove}
  0\leq\max_{D'\in\mathcal{D}_{\mu{m}}}\max_{\eta\in{E}_\mu}\inf_{f\in\mathcal{V}_{\eta}[\nu,\mu;D']}
  \|f\|_\mathcal{H}^2 < \infty. 
 \end{align}
 The above argument leads to 
 \begin{align*}
  \min_{f\in\mathcal{V}_{\eta}[\nu,\mu;D']} \|f\|_{\mathcal{H}}^2
  \leq 
  \min_{f\in\widetilde{\mathcal{V}}_{\eta}[\nu,\mu;D']}\|f\|_{\mathcal{H}}^2<\infty
 \end{align*}
 for any $\eta\in{E}_\mu$. 
 Let us define 
 $\mathcal{C}[D]=\mathrm{conv}\{k(\cdot,x_i)\,:\,i\in{I}_{+}\}\ominus\mathrm{conv}\{k(\cdot,x_i)\,:\,i\in{I}_{-}\}$ 
 for the original dataset~$D$. Then, the inclusion relation
 $\widetilde{\mathcal{U}}_{\eta}^{\pm}[\nu,\mu;D']\subset\mathrm{conv}\{k(\cdot,x_i)\,:\,i\in{}I_{\pm}\}$ 
 leads to 
 $\widetilde{\mathcal{V}}_{\eta}[\nu,\mu;D']\subset\mathcal{C}[D]$. 
 Hence, we obtain 
 \begin{align*}
  \mathrm{opt}(\nu,\mu;D')
  &=
  \max_{\eta\in{E}_\mu}\min_{f\in\mathcal{V}_{\eta}[\nu,\mu;D']} \|f\|_{\mathcal{H}}^2 \\
  &\leq
  \max_{\eta\in{E}_\mu}\min_{f\in\widetilde{\mathcal{V}}_{\eta}[\nu,\mu;D']}  \|f\|_{\mathcal{H}}^2 \\
  &\leq\ 
  \max_{\eta\in{E}_\mu}\max_{f\in\widetilde{\mathcal{V}}_{\eta}[\nu,\mu;D']} \|f\|_{\mathcal{H}}^2 \\
  &\leq\ 
  \max_{f\in\mathcal{C}[D]}\|f\|_{\mathcal{H}}^2 \\
  &<\infty. 
 \end{align*}
 The boundedness of $\max_{f\in\mathcal{C}[D]}\|f\|_{\mathcal{H}}^2$ comes from the compactness of
 $\mathcal{C}[D]$ and the continuity of the norm. 
 More precisely, it is bounded above by twice the maximum eigenvalue of the Gram matrix defined from the
 non-contaminated data $D$. 
 The upper bound does not depend on the contaminated dataset $D'\in\mathcal{D}_{\mu{m}}$. 
 Thus, the first inequality of \eqref{eqn:inequalities_tobe_prove} holds. 
\end{proof}

\begin{lemma}
 \label{lemma:unbounded-breakdown-point-ER-SVM-1}  
 Under the condition of Theorem~\ref{theorem:breakdown-point-optvalue}, 
 we assume $\nu-\mu>2(r-2\mu)$. Then, we have 
 \begin{align*}
  \sup\{\mathrm{opt}(\mu,\nu;D'): {D'\in\mathcal{D}_{\mu{m}}}\}=\infty. 
 \end{align*}
\end{lemma}
\begin{proof}
 [Proof of Lemma~\ref{lemma:unbounded-breakdown-point-ER-SVM-1}]
 We use the same notation as in the proof of Lemma~\ref{lemma:breakdown-point-ER-SVM}. 
 Without loss of generality, we assume $r=|I_{-}|/m$. 
 The parameter $\mu$ is expressed as $\mu=r/(c+2)$ for $c>0$. 
 We prove that there exists a feasible parameter $\eta\in{E}_\mu$ and a contaminated 
 training set $D'=\{(x_i',y_i'):i\in[m]\}\in\mathcal{D}_{\mu{m}}$ such that 
 $\mathcal{U}_\eta^{-}[\nu,\mu;D']=\emptyset$. 
 The construction of the dataset $D'$ is illustrated in Figure~\ref{fig:Plot-Lemma_opt_infinity}. 
 Suppose that $|\widetilde{I}_{+}|=0$ and $|\widetilde{I}_{-}|=\mu{m}$ and 
 that $y_i'=+1$ holds for all $i\in\widetilde{I}_{-}$, meaning that all outliers in $D'$ are made by flipping the labels
 of the negative samples in $D$. This is possible, because $\mu{m}<|I_{-}|/2<|I_{-}|$ holds. The inequality
 $\nu-\mu>2(r-2\mu)$ leads to 
 \begin{align}
  \label{eqn:Ineg_bound}
  1>\frac{2}{(\nu-\mu)m}\cdot\frac{c}{c+2}|I_{-}|. 
 \end{align}
 The outlier indicator $\eta'=(\eta_1',\ldots,\eta_m')\in{E_\mu}$ is defined by $\eta_i'=0$ for $\mu{m}$ samples in
 $I_{-}\setminus\widetilde{I}_{-}$, and $\eta_i'=1$ otherwise. This assignment is possible because 
 \begin{align*}
  |I_{-}\setminus\widetilde{I}_{-}|
  =|I_{-}|-\mu{m}  
  =\frac{c+1}{c+2}|I_{-}|
  >  
  \frac{1}{c+2}|I_{-}|
  =
  \mu{m}. 
 \end{align*}
 Then, we have 
 \begin{align}
  \label{proof:eq-1}
  |\{i\in{}I_{-}\setminus\widetilde{I}_{-}\,:\,\eta_i'=1\}|
 & =
  |I_{-}\setminus\widetilde{I}_{-}|-\mu{m} \nonumber\\
&  =
  |I_{-}|-2\mu{m} \nonumber\\
&  =
  \frac{c}{c+2}|I_{-}|. 
 \end{align}
 From \eqref{eqn:Ineg_bound} and \eqref{proof:eq-1}, we have
 \begin{align*}
  1>\frac{2}{(\nu-\mu)m}
  |\{i\in{}I_{-}\setminus\widetilde{I}_{-}\,:\,\eta_i'=1\}|. 
 \end{align*}
 In addition, $y_i'=-1$ holds only when $i\in{}I_{-}\setminus{}\widetilde{I}_{-}$. 
 Therefore, we have $\mathcal{U}_{\eta'}^{-}[\nu,\mu;D']=\emptyset$. 
 The infeasibility of the dual problem means the unboundedness of the primal problem. 
 Hence, there exists a contaminated dataset $D'\in\mathcal{D}_{\mu{m}}$ and
 an outlier indicator $\eta'\in{}E_\mu$ such that 
 \begin{align*}
  \mathrm{opt}(\mu,\nu;D')\geq \min_{f\in\mathcal{V}_{\eta'}[\nu,\mu;D']}\|f\|_{\mathcal{H}}^2=\infty
 \end{align*}
 holds. 
\end{proof}
\begin{figure}[t]
 \begin{center}
  \includegraphics[scale=0.8]{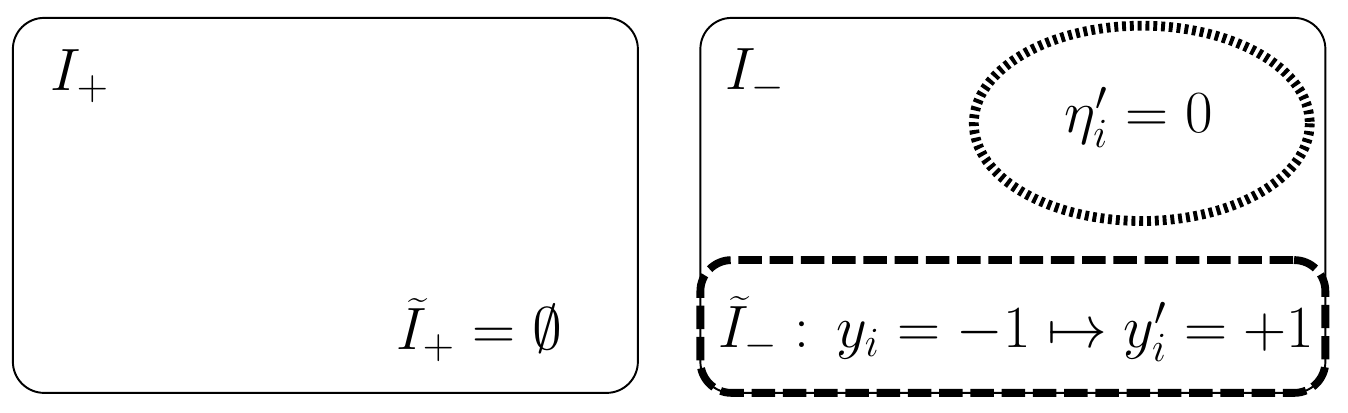}
 \caption{Index sets $\widetilde{I}_{\pm}$ and value of $\eta_i'$ defined in the proof of Lemma~\ref{lemma:unbounded-breakdown-point-ER-SVM-1}.} 
 \label{fig:Plot-Lemma_opt_infinity}
 \end{center}
\end{figure}

\subsection{Proof of Theorem~\ref{theorem:breakdown-point_upper_bound}}
\label{appendix:upperbound_breakdown_opt}
\begin{proof}
 For a rational number $\mu\in(0,1/4)$, there exists an $m\in\Nbb$ such that 
 $\mu{m}\in\Nbb$ and $2\mu{m}+1\leq{}m-(2\mu{m}+1)$ hold. 
 For such $m$, let $D=\{(x_i,y_i):i\in[m]\}$ be a training data such that 
 $|I_{-}|=2\mu{m}+1$ and $|I_{+}|=m-(2\mu{m}+1)$, where
 the index sets $I_{\pm}$ are defined in the proof of Appendix~\ref{appendix:proof_breakdown_obj}. 
 Since the label ratio of $D$ is $r=\min\{|I_{-}|,|I_{+}|\}/m=2\mu+1/m$, and we have $\mu<r/2$. 
 For $\mathcal{D}_{\mu{m}+1}$ defined from $D$, 
 let $D'=\{(x_i',y_i'):i\in[m]\}\in\mathcal{D}_{\mu{m}+1}$ be a contaminated dataset of $D$ such that 
 $\mu{m}+1$ outliers are made by flipping the labels of the negative samples in $D$. 
 Thus, there are $\mu{m}$ negative samples in $D'$. Let us define the outlier indicator 
 $\eta'=(\eta_1',\ldots,\eta_m')\in{E_\mu}$ such that $\eta_i'=0$ for $\mu{m}$ negative samples in
 $D'$. Then, any sample in $D'$ with $\eta_i'=1$ should be a positive one. Hence, 
 we have $\mathcal{U}_{\eta'}^{-}[\nu,\mu;D']=\emptyset$. 
 The infeasibility of the dual problem means that the primal problem is unbounded. 
 Thus, we obtain $\mathrm{opt}(\nu,\mu;D')=\infty$. 
\end{proof}

\subsection{Proof of Theorem~\ref{theorem:bounded_bias}} 
\label{appendix:proof_breakdown_bias}
Let us define $f_D+b_D$ with $f_D\in\mathcal{H},\,b_D\in\Rbb$ as the decision function estimated using robust $(\nu,\mu)$-SVM
based on the dataset~$D$. 

\begin{proof}
 The non-contaminated dataset is denoted as $D=\{(x_i,y_i):i\in[m]\}$. 
 For the dataset $D$, let $I_+$ and $I_{-}$ be the index sets defined by $I_{\pm}=\{i:y_i=\pm1\}$. 
 Under the conditions of Theorem~\ref{theorem:bounded_bias}, 
 inequality \eqref{eqn:key_inequality} holds.  
 Given a contaminated dataset $D'=\{(x_i',y'_i):i\in[m]\}\in\mathcal{D}_{\mu{m}-\ell}$, 
 let $r_i'(b)$ be the negative margin of $f_{D'}+b$, i.e., $r_i'(b)=-y_i'(f_{D'}(x_i')+b)$ for
 $(x_i',y_i')\in{D'}$. For $b\in\Rbb$, the function $\zeta(b)$ is defined as 
 \begin{align*}
  \zeta(b)=\frac{1}{m}\sum_{i\in{T_b}}r_i'(b), 
 \end{align*}
 where the index set $T_b$ is given by
 \begin{align*}
  T_b=\{\sigma(j)\in[m]\,:\,\mu{m}+1\leq{j}\leq\nu{m}\}
 \end{align*}
 for the sorted negative margins, $r_{\sigma(1)}'(b)\geq\cdots\geq{}r_{\sigma(m)}'(b)$. 
 For simplicity, we drop the dependency of the permutation $\sigma$ on $b$. 
 The estimated bias term $b_{D'}$ is the optimal solution of $\zeta(b)$ because of \eqref{eqn:diff_CVaR}. 
 The function $\zeta(b)$ is continuous. In addition, $\zeta(b)$ is linear on the interval such that $T_b$ is unchanged. 
 Hence, $\zeta(b)$ is a continuous piecewise linear function. 
 Below, we prove that the minimum solution of
 $\zeta(b)$ is bounded regardless of the contaminated dataset $D'\in\mathcal{D}_{\mu{m}-\ell}$. 

 For the non-contaminated data $D$, let $R$ be a positive real number such that
 \begin{align*}
  \sup\{|f_{D''}(x)|:(x,y)\in{D},\, D''\in\mathcal{D}_{\mu{m}-\ell}\}\leq{R}. 
 \end{align*}
 The existence of $R$ is guaranteed. Indeed, one can choose 
 \begin{align*}
  R=\sup_{D''\in\mathcal{D}_{\mu{m}-\ell}}\|f_{D''}\|_{\mathcal{H}}\cdot\max_{(x,y)\in{D}}\sqrt{k(x,x)}<\infty, 
 \end{align*}
 because the RKHS norm of $f_{D''}$ is uniformly bounded above for $D''\in\mathcal{D}_{\mu{m}-\ell}$ and $D$ is a finite
 set. For the contaminated dataset $D'=\{(x_i',y_i')\,:\,i\in[m]\}\in\mathcal{D}_{\mu{m}-\ell}$, let us define the index sets 
 $I_{\pm}',I_{\mathrm{in},\pm}^{'}$ and $I_{\mathrm{out},\pm}^{'}$ for each label by 
\begin{align*}
 I_{\pm}'&=\{i\in[m]\,:\,y_i'=\pm1\},\\
 I_{\mathrm{in},\pm}'&=\{i\in{}I_{\pm}' \,:\, |f_{D'}(x_i')|\leq{}R\}, \\ 
 I_{\mathrm{out},\pm}'&=\{i\in{}I_{\pm}'\,:\, |f_{D'}(x_i')|>R\}. 
\end{align*}
For any non-contaminated sample $(x_i,y_i)\in{D}$, we have $|f_{D'}(x_i)|\leq{R}$. 
Hence, $(x_i',y_i')\in{D'}$ for $i\in{I}_{\mathrm{out},\pm}'$ should be an outlier that is not included in $D$.  This
 fact leads to  
\begin{align*}
& |I_{\mathrm{out},+}'|+ |I_{\mathrm{out},-}'|\,\leq\,\mu{m}-\ell,\\
& |I_{\mathrm{in},\pm}'|\,\geq\,|I_{\pm}|-(\mu{m}-\ell)\geq (r-\mu)m+\ell. 
\end{align*}

Based on the argument above, we prove two propositions: 
\begin{enumerate}
 \item The function $\zeta(b)$ is increasing for $b>R$. 
 \item The function $\zeta(b)$ is decreasing for $b<-R$. 
\end{enumerate}
 In addition, for any $D'\in\mathcal{D}_{\mu{m}-\ell}$, the Lipschitz constant of
 $\zeta(b)$ is greater than or equal to $1/m$ for $R<|b|$. 

Let us prove the first statement. If $b>R$ holds, we have 
\begin{align}
\label{eqn:CVaR_b_pos}
 R-b <  \min\{r_i'(b):i\in{I_{\mathrm{in},-}'}\}
\end{align}
from the definition of the index set $I_{\mathrm{in},-}'$. 
Let us consider two cases: 
 (i) for all $i\in{T_b}$, $R-b<r_i'(b)$ holds, and 
 (ii) there exists an index $i\in{}T_b$ such that $r_i'(b)\leq{R-b}$. 

For a fixed $b$ such that $b>R$, let us assume (i) above. 
Then, for any index $i$ in ${I}_{+}'\cap{T_b}$, we have $R<-f_{D'}(x_i')$, meaning that
$i\in{}I_{\mathrm{out},+}'$. 
Hence, the size of the set ${I}_{+}'\cap{T_b}$ is less than or equal to $\mu{m}-\ell$. 
Therefore, the size of the set ${I}_{-}'\cap{T_b}$ is greater than or equal to
$(\nu-\mu)m-(\mu{m}-\ell)=(\nu-2\mu)m+\ell$. The first inequality of \eqref{eqn:breakdown_point_bias} leads to
$(\nu-2\mu)m+\ell>\mu{m}-\ell$. 
Therefore, in the set $T_b$, the number of negative samples is more than the number of positive samples. 

For a fixed $b$ such that $b>R$, let us assume (ii) above. 
Due to the inequality~\eqref{eqn:CVaR_b_pos}, for any index $i\in{}I_{\mathrm{in},-}'$, 
the negative margin $r_i'(b)$ is at the top $\nu{m}$ of those ranked in the descending order. 
Hence, the size of the set ${I}_{-}'\cap{T_b}$ is greater than or equal to
$|I_{\mathrm{in},-}'|-\mu{m}\geq(r-2\mu)m$. 
Therefore, the size of the set ${I}_{+}'\cap{T_b}$ is less than or equal to $(\nu-\mu)m-(r-2\mu)m=(\nu-r+\mu)m$. 
The second inequality of \eqref{eqn:breakdown_point_bias} leads to $(\nu-r+\mu)m<(r-2\mu)m$. 
Also in the case of (ii), the negative label dominates the positive label in the set $T_b$. 

For negative (resp. positive) samples, the negative margin is expressed as $r_i'(b)=u_i+b$ (resp. $r_i'(b)=u_i-b$) 
with a constant $u_i\in\mathbb{R}$. Thus, the continuous piecewise linear function $\zeta(b)$ is expressed as 
\begin{align*}
 \zeta(b)=\frac{v_{b}+b\cdot{}w_b}{m}, 
\end{align*}
 where $v_{b},w_{b}\in\Rbb$ are constants as long as $T_b$ is unchanged. 
 As proved above, $w_b$ is a positive integer, 
 since negative samples are more than positive samples in $T_b$, when $b>R$. 
 As a result, the optimal solution of the bias term should satisfy 
 \begin{align*}
  \sup\{b_{D'}:D'\in\mathcal{D}_{\mu{m}-\ell}\} \leq R. 
 \end{align*}
 In the same manner, one can prove the second statement by using the fact that $b<-R$ is a sufficient condition of 
 \begin{align*}
  R+b <   \min\{r_i'(b): {i\in{I_{\mathrm{in},+}'}}\}. 
 \end{align*}
 Then, we have
 \begin{align*}
  \inf\{b_{D'}:D'\in\mathcal{D}_{\mu{m}-\ell}\} \geq -R. 
 \end{align*}
 In summary, we obtain
 \begin{align*}
  \sup\{|b_{D'}| : {D'\in\mathcal{D}_{\mu{m}-\ell}}\}
  \leq{}R<\infty. 
 \end{align*}
\end{proof}

\subsection{Proof of Theorem~\ref{theorem:bounded_kernel_bounded_bias}}
\label{appendix:proof_bounded_kernel_breakdown_bias}
\begin{proof}
We use the same notation as in the proof of Theorem~\ref{theorem:bounded_bias} in Appendix~\ref{appendix:proof_breakdown_bias}. 
Note that inequality \eqref{eqn:key_inequality} holds under the assumption of Theorem~\ref{theorem:bounded_kernel_bounded_bias}. 
The reproducing property of the RKHS inner 
product yields 
\begin{align*}
 |f_{D'}(x_i')|
 \,\leq\,
 \|f_{D'}\|_\mathcal{H}\sqrt{k(x_i',x_i')}
 \,\leq\,
 \sup_{D''\in\mathcal{D}_{\mu{m}}}\!\!\|f_{D''}\|_\mathcal{H}\cdot\sup_{x\in\mathcal{X}}\sqrt{k(x,x)}<\infty
\end{align*}
 for any $D'=\{(x_i',y_i')\,:\,i\in[m]\}\in\mathcal{D}_{\mu{m}}$ due to the boundedness of
 the kernel function and inequality~\eqref{eqn:key_inequality}. 
 Hence, for a sufficiently large $R\in\Rbb$, the sets $I_{\mathrm{out},+}'$ and $I_{\mathrm{out},-}'$ become empty for
 any $D'\in\mathcal{D}_{\mu{m}}$. 

Under inequality \eqref{eqn:CVaR_b_pos}, suppose that $R-b<r_i'(b)$ holds for all 
$i\in{T_b}$. Then, for $i\in{I}_{+}'\cap{T_b}$, we have $R<-f_{D'}(x_i')$. Thus, $i\in{}I_{\mathrm{out},+}'$ holds. 
Since $I_{\mathrm{out},+}'$ is the empty set, ${I}_{+}'\cap{T_b}$ is also the empty set. 
Therefore, ${T_b}$ has only negative samples. 
Let us consider the other case; i.e., there exists an index $i\in{}T_b$ such that $r_i'(b)\leq{R-b}$. 
 Assuming that $\nu-\mu<2(r-2\mu)$, one can prove that the negative labels dominate the positive labels in $T_b$ in the same
 manner as the proof of Theorem~\ref{theorem:bounded_bias}. 
Eventually, for any $D'\in\mathcal{D}_{\mu{m}}$, 
the function $\zeta(b)$ is strictly increasing for $b>R$. 
In the same way, one can prove that $\zeta(b)$ is strictly decreasing for $b<-R$. 
Moreover, for any $D'\in\mathcal{D}_{\mu{m}}$ and for $|b|>R$, 
one can prove that the absolute value of the slope of $\zeta(b)$ is bounded below by
$1/m$ according to the argument in the proof of Theorem~\ref{theorem:bounded_bias}. 
As a result, we obtain 
$\sup\{|b_{D'}|:D'\in\mathcal{D}_{\mu{m}}\}\leq{}R$. 
\end{proof}




\bibliographystyle{plain}

\end{document}